\documentclass[11pt]{article}

\usepackage[letterpaper,margin=1in]{geometry} 
\usepackage{setspace}
\usepackage[utf8]{inputenc}
\usepackage{graphicx} 
\usepackage[english]{babel}

\usepackage{float}
\usepackage{subfigure}
\usepackage{amsfonts}
\usepackage{amsmath}
\usepackage{amssymb}
\usepackage{graphicx}
\usepackage{algpseudocode}
\usepackage{amsthm}
\usepackage{physics}
\theoremstyle{definition}
\newtheorem{definition}{Definition}[section]
\newtheorem{theorem}{Theorem}[section]

\newtheorem{lemma}[theorem]{Lemma}

\usepackage{bbm}
\usepackage[dvipsnames]{xcolor}
\usepackage[pagebackref]{hyperref}
\hypersetup{
colorlinks=true,
urlcolor=blue,
linkcolor=blue,
citecolor=OliveGreen,
}
\usepackage[capitalise,nameinlink]{cleveref}
\crefname{lemma}{Lemma}{Lemmas}
\crefname{fact}{Fact}{Facts}
\crefname{theorem}{Theorem}{Theorems}

\usepackage{amsmath,enumitem,amsthm,float}
\let\oldcap\cap
\let\cap\oldcap
\usepackage[linesnumbered,algoruled,boxed,lined]{algorithm2e}

\usepackage{graphicx}		
\usepackage{amssymb}
\usepackage{amsmath}
\usepackage{amsthm}
\usepackage{mathtools}
\usepackage{enumitem}
\usepackage{verbatim}
\usepackage{xcolor}
\usepackage[utf8]{inputenc}
\usepackage[T1]{fontenc}
\usepackage{microtype}
\usepackage[bottom]{footmisc}

\newcommand{\E}{\mathop{\mathbb{E}}}
\newcommand{\ex}[2]{{\ifx&#1& \mathbb{E} \else \underset{#1}{\mathbb{E}} \fi \left[#2\right]}}
\renewcommand{\var}[2]{{\ifx&#1& \mathsf{Var} \else \underset{#1}{\mathsf{Var}} \fi \left[#2\right]}}

\newcommand{\D}{\mathcal{D}}

\newcommand{\Z}{\mathcal{Z}}

\newcommand{\sZ}{\tilde{Z}}
\newcommand{\sz}{\tilde{z}}

\newcommand{\1}{\mathbbm{1}}
\newcommand{\simiid}{\overset{\text{i.i.d.}}{\sim}}
\newcommand{\Unif}[1]{\mathrm{Unif}\left( #1\right)}
\newcommand{\sett}[1]{\left\{ #1 \right\}}

\let\originalleft\left
\let\originalright\right
\renewcommand{\left}{\mathopen{}\mathclose\bgroup\originalleft}
\renewcommand{\right}{\aftergroup\egroup\originalright}

\newcommand{\mybignote}[2]{}

\title{The optimal information complexity of VC Learning}

\author{ 
   Steve Hanneke\thanks{Department of Computer Science, Purdue University; \texttt{steve.hanneke@gmail.com.}}
   \and Juexiao Wang\thanks{Department of Computer Science, Purdue University; \texttt{wang5360@purdue.edu.}}
}

\date{}
\allowdisplaybreaks
\begin{document}
\maketitle
\begin{abstract}
Steinke and Zakynthinou \cite{pmlr-v125-steinke20a} introduces the Conditional Mutual Information (CMI) framework of analyzing the information complexity of learning algorithms based on algorithm-dependent information-theoretic quantities.
We study one of these quantities, the evaluated Conditional Mutual Information (eCMI).
It has been an interesting question whether the optimal PAC guarantee for VC classes can be recovered from the algorithm-dependent analyses via CMI.
And we show that it is possible to recover this guarantee by constructing a learning algorithm whose eCMI is of order $O(d)$ in the realizable case, where $d$ is the VC-dimension of the concept class. Specially, our algorithm is a randomized Majority-of-5 base learners with optimal in-expectation generalization guarantee.
\end{abstract}
\newpage

{
\hypersetup{linkcolor=blue,
citecolor=red}   
\tableofcontents
}
\newpage
\section{Introduction}
The probably approximately correct (PAC) learning model defined by Vapnik
and Chervonenkis \cite{vapnik1974theory} and Valiant \cite{valiant1984theory} is the most classic theoretical framework for studying classification problems in supervised learning. In this framework, a learning algorithm receives a training data set of $n$ i.i.d. samples $\{(X_i,Y_i)\}_{i=1}^n$ drawn from an unknown distribution $\mathcal{D}$ over $\mathcal{Z}=\mathcal{X}\times\{0,1\}$. The distribution is assumed to be realizable with respect to a concept class $\mathcal{H}\subseteq \{0,1\}^{\mathcal{X}}$ if $Y=h^*(X)$ for some $h^*\in \mathcal{H}$. The goal is to minimize the probability that the output classifier makes a mistake on predicting the label of a fresh sample drawn from the unknown distribution $\mathcal{D}$.

PAC learnability is equivalent to the finiteness of a combinatorial measure, known as VC-dimension. 
Classical work focuses on the empirical risk minimization (ERM) algorithm, for which uniform convergence provides an upper bound on its expected generalization error guarantee by ensuring that the empirical error rate of a hypothesis $h\in \mathcal{H}$ converges to its true error rate uniformly over the class $\mathcal{H}$. 
Nevertheless, the ERM algorithm can be suboptimal in some cases. 
In fact, the optimal expected error rate for PAC learning in the realizable case is of order $O(\frac{d}{n})$, where $d$ denotes the VC-dimension of a concept class $\mathcal{H}$ and $n$ the sample size.
The proof of this is by a direct analysis of learning algorithms. \cite{HAUSSLER1994248} provides an optimal guarantee on the expected generalization error  and \cite{JMLR:v17:15-389} provides an optimal guarantee on the high probability generalization error.

\paragraph{Generalization with Mutual Information.}
Recent literature focuses on deriving generalization guarantees through algorithm-dependent analyses with an emphasis on the information complexity of learning algorithms for PAC learning a concept class. 
In particular, \cite{pmlr-v51-russo16,NIPS2017_ad71c82b} bound the expected difference between the empirical and true loss of a learning algorithm with the mutual information between the input training samples and the output hypothesis, denoted by $I(A(S);S)$, which quantifies how much information the output hypothesis of a learning algorithm reveal about its input training samples. A learning algorithm with low $I(A(S);S)$ is guaranteed to generalize well.
However, $I(A(S);S)$ can be unbounded for some choice of concept class and distribution.
Specifically, \cite{pmlr-v75-nachum18a,pmlr-v83-bassily18a} show that any proper consistent learner for class of threshold functions of VC-dimension 1 has input-output mutual information $\Omega(\log\log |\mathcal{Z}|)$ for some realizable distribution. 

\paragraph{Conditional Mutual Information and Related Work.}
Steinke and Zakynthinou \cite{pmlr-v125-steinke20a} proposed to normalize the information content of each sample to one bit with Conditional Mutual Information (CMI). 
In line with the symmetrization argument in analyzing the generalization guarantee of the ERM algorithm, the CMI framework conditions on a supersample $\sZ$ consisting of $n$ training samples and $n$ ghost samples drawn i.i.d. from the distribution. Then it considers the mutual information between a uniform random selection $\boldsymbol{\sigma}$ of the input training samples $\sZ_{\boldsymbol{\sigma}}$ and the output hypothesis $A(\sZ_{\boldsymbol{\sigma}})$ of an algorithm.
CMI of a learning algorithm is at most $n$ bits, corresponding to an algorithm that reveals all of its input induced by $\boldsymbol{\sigma}$ over $n$ independent swap variables.
A variant of the CMI framework is the evaluated CMI. In contrast, eCMI measures the information that the algorithm reveals about its input training samples through the loss vector $\vec{\ell}(h,\sZ)$ evaluated on the supersample, where $h$ denotes the output hypothesis of the algorithm.

CMI and other closely related techniques have been used to analyze many other different algorithms such as stochastic gradient descent, Langevin dynamics, stochastic convex optimization \cite{NEURIPS2020_712a3c98, pmlr-v134-neu21a,10.5555/3692070.3692151}, etc.
It remains an interesting \textit{open question} whether the optimal PAC guarantee for VC classes can be recovered from the algorithm-dependent analyses via CMI.
In particular, \cite{pmlr-v125-steinke20a} shows that for every concept class of VC-dimension $d$, there exists an ERM algorithm with CMI at most $d\log n +2$ bits and conjectures that there exists a learning algorithm with CMI of order $O(d)$, which would entail a tight generalization guarantee via CMI.

A line of work has made interesting progress towards this question.
\cite{NEURIPS2021_ddbc86dc} show that it is not possible to find a proper learning algorithm for every VC class that removes the $\log n$ factor.
Additionally, they show that the sublinear growth of the evaluated CMI characterizes the learnability by consistent algorithms. 
\cite{isit9834400} show that the leave-one-out variants of the CMI framework characterizes the expected generalization guarantee of learning algorithms. Specifically, they recover the optimal guarantee for learning VC classes in the realizable case via $\text{LOO}^{e}\text{CMI}_{D}$ through the one-inclusion-graph algorithm \cite{HAUSSLER1994248}.
However, the leave-one-out eCMI has not been found in many applications, perhaps because it is restricted to sample-consistent learners.

\paragraph{Main Result.}
In this paper, we resolve this question via eCMI, for which an optimal guarantee has not been shown.
The main contribution (Theorem \ref{thm:eCMI(A)}, Theorem \ref{thm:eCMI(maj-of-3)}) of our work is to present a randomized Majority-of-5 learner that achieves an optimal worst-case eCMI of order $O(d)$ and the expected training error rate of order $O(\frac{d}{n})$, thereby recovering the optimal expected generalization guarantee in the realizable case. More generally, we show that for any base learner with expected error rate $R(m)$,
a randomized Majority-of-5 of such base learners would have a guarantee on the worst-case eCMI of order $O(nR(\frac{n}{5}))$ and its expected training error rate is at most $5\cdot R(\frac{n}{5})$. We also study the distributional eCMI and present a simpler algorithm, a Majority-of-3 base learners with optimal expected error rate, that achieves the optimal guarantee on distributional eCMI.

Our algorithm performs a uniform random permutation over the $n$ training data points, then it runs the base learner on five disjoint pieces of the training data points of equal size, after adjusting for duplicates on each to emulate i.i.d. sampling from the empirical distribution on $n$ training data points.
Finally, it combines the predictions of the five output classifiers with a majority vote. 

Note that our constructed randomized Majority-of-5 learner is improper \cite{pmlr-v125-bousquet20a} as it may output any classifier, not necessarily in the concept class, which as mentioned above is known to be unnecessary. 
While it remains open whether there exists a learning algorithm with an optimal guarantee on CMI, 
we find eCMI easier to work with. For instance, any ERM has $O(d\log n)$ guarantee on eCMI while it takes a special ERM algorithm that break ties based on well-ordering the concept class to achieve $O(d\log n)$ guarantee on CMI \cite{pmlr-v125-steinke20a}.
Moreover, we discuss in Section \ref{Gen_from_eCMI} that an optimal guarantee on eCMI is sufficient for optimal generalization guarantees. 

\paragraph{Overview of Techniques.}
The starting point of our analysis is to bound the eCMI of a randomized learning algorithm with its conditional entropy of errors on $2n$ data points.
Inspired by the Majority-of-3 learner in \cite{pmlr-v247-aden-ali24a}, in which independent base classifiers reduces the pointwise probability of error, we partition the training data points into disjoint pieces to have conditionally independent base classifiers given the randomness of the algorithm, such that the pointwise conditional probability of error is reduced, thereby reducing the conditional binary entropy of errors.

In addition, we divide the $2n$ data points into conditional error regions based on the conditional probability of error of the base learner such that the guarantee on $R(m)$ of the base learner yields a geometrically increasing sequence of expected size of conditional error regions.
By having five such base learners, for any data point in the $k$-th part of the partition, at which the $k$-th base learner might not guarantee its conditional probability of error over random swaps, we have at least two of the other four base classifiers err when the majority vote classifier errs. 
Thus, the conditional probability of error of the majority vote learner can be bounded by a squared conditional probability of error due to the conditional independence of base classifiers, which suffices to compensate for the exponential growth rate of the sequence of expected size of conditional error regions.
If the base learner has an optimal expected error rate of order $O(\frac{d}{n})$, then it yields a series of conditional binary entropies of errors that converges to a constant $d$.
We elaborate on this approach in Section \ref{Our apporach}.


\section{Preliminaries}\label{Preliminaries}
In this section, we introduce the definition of the evaluated conditional mutual information of a learning algorithm and present the bound on the expected generalization error in terms of eCMI.

\vspace{-6pt}
\subsection{Evaluated CMI}
\begin{definition}(Evaluated Conditional Mutual Information (eCMI) of an Algorithm)
Let $A:\mathcal{Z}^n\rightarrow \{0,1\}^{\mathcal{X}}$ be a randomized or deterministic algorithm. Let $\mathcal{D}$ be a realizable probability distribution on $\mathcal{Z}$ and let $\sZ\in \mathcal{Z}^{2n}$ be a realizable supersample consisting of $2n$ i.i.d. samples from $\mathcal{D}$. Let $\boldsymbol{\sigma} = (\boldsymbol{\sigma}(1), \boldsymbol{\sigma}(2), \dots, \boldsymbol{\sigma}(n))$ be independent uniform random variables with $\boldsymbol{\sigma}(i)\sim \text{Unif}(\{i,i+n\})$ for all $i\in[n]$. The independent random swaps $\boldsymbol{\sigma}$ are independent of $\tilde{Z}$ and the randomness of $A$. 
Define $\sZ_{\boldsymbol{\sigma}} = \{ (X_{\boldsymbol{\sigma}(1)}, Y_{\boldsymbol{\sigma}(1)}), \dots, (X_{\boldsymbol{\sigma}(n)}, Y_{\boldsymbol{\sigma}(n)})\} \in \mathcal{Z}^n$ as the $n$ training samples selected by $\boldsymbol{\sigma}$, we also refer to it as training data variables.
Let $\ell(h,(x,y)) = \1[h(x)\neq y]$ be a 0-1 loss function.
Define $\vec{\ell}:\{0,1\}^{\mathcal{X}} \times \mathcal{Z}^{2n} \rightarrow \sett{0,1}^{2n}$ as vector of 0-1 losses on $2n$ samples defined by
$\vec{\ell}(A(\sZ_{\boldsymbol{\sigma}}),\tilde{Z})_{i} = \ell(A(\sZ_{\boldsymbol{\sigma}}), \sZ_{i})$ for all $i\in[2n]$.

\noindent
The evaluated conditional mutual information of $A$ with respect to the distribution $\mathcal{D}$ is
\[
\text{eCMI}_{\mathcal{D}}(\ell(A)) = I(\vec{\ell}(A(\sZ_{\boldsymbol{\sigma}}),\tilde{Z}); \boldsymbol{\sigma}|\tilde{Z}).
\]
The distribution-free, also known as worst-case evaluated conditional mutual information, of $A$ on a worst-case supersample is
\[
\text{eCMI}(\ell(A)) = \sup_{\substack{\sz\in \mathcal{Z}^{2n} \text{realizable}}} I(\vec{\ell}(A(\sz_{\boldsymbol{\sigma}}),\sz); \boldsymbol{\sigma}).
\]
And we have $\text{eCMI}_{\mathcal{D}}(\ell(A)) \leq \text{eCMI}(\ell(A))$ for any algorithm $A$ and distribution $\mathcal{D}$.
\end{definition}

\vspace{-12pt}
\subsection{Generalization from eCMI}\label{Gen_from_eCMI}
Theorem 5.8 of \cite{pmlr-v125-steinke20a} derives a generalization bound of a learning algorithm in terms of CMI, assuming that the training error rate of the algorithm is not necessarily zero. 
By slight modifications to its proof, it can be extended to a generalization bound in terms of $\text{eCMI}_D(A)$, as follows.

\begin{theorem}\label{thm:cmigen-realize2}
Let $\mathcal{D}$ be a distribution on $\mathcal{Z}$. Let $A :\mathcal{Z}^n \to \{0,1\}^{\mathcal{X}}$ be a randomized algorithm. Let $\ell : \{0,1\}^{\mathcal{X}} \times \mathcal{Z} \to [0,1]$ be an arbitrary bounded function.
Then $$\ex{Z \gets \D^n, A}{\ell(A(Z),\mathcal{D})} \le  2 \cdot \ex{Z \gets \D^n, A}{\ell(A(Z),Z)} + \frac{3 \cdot \text{eCMI}_D(\ell (A))}{n}.$$
\end{theorem}
The proof of Theorem 5.8 of \cite{pmlr-v125-steinke20a} begins by expressing the expected difference between the true error rate and the training error rate of the output hypothesis of the algorithm $A$ in terms of the expected difference between its empirical error rate on the ghost samples and training samples.
Then it applies Lemma 2.11 of \cite{pmlr-v125-steinke20a} (known as "Donsker-Varadhan dual characterization of KL divergence" or "Gibbs variational principle") to the expected difference between the empirical error rate of the output hypothesis of $A$ on the ghost and training samples: $f_{\sZ}(A(\sZ_{\boldsymbol{\sigma}}),\boldsymbol{\sigma}) = l(A(\sZ_{\boldsymbol{\sigma}}),\sZ_{\bar{\boldsymbol{\sigma}}}) - u\cdot l(A(\sZ_{\boldsymbol{\sigma}}),\sZ_{\boldsymbol{\sigma}})$ where $u > 0$, yielding an expected generalization guarantee in terms of $\text{CMI}_{D}(A)$.
The proof of the generalization bound in terms of eCMI is nearly identical, except that we replace the function $f_{\sZ}$ with $g_{\sZ}(A(\sZ_{\boldsymbol{\sigma}})(\sZ),\boldsymbol{\sigma}) = l(A(\sZ_{\boldsymbol{\sigma}})(\sZ_{\bar{\boldsymbol{\sigma}}}),\sZ_{\bar{\boldsymbol{\sigma}}}) - u
\cdot l(A(\sZ_{\boldsymbol{\sigma}})(\sZ_{\boldsymbol{\sigma}}),\sZ_{{\boldsymbol{\sigma}}})$, where $A(\sZ_{\boldsymbol{\sigma}})(\sZ)$ is the projection of the output hypothesis of $A$ onto the supersample $\sZ$. Since there exists a bijection between $A(\sZ_{\boldsymbol{\sigma}})(\sZ)$ and $\vec{\ell}(A(\sZ_{\boldsymbol{\sigma}}),\tilde{Z})$, we have $I(A(\sZ_{\boldsymbol{\sigma}})(\sZ); {\boldsymbol{\sigma}}|\sZ) = I(\vec{\ell}(A(\sZ_{\boldsymbol{\sigma}}),\tilde{Z}); {\boldsymbol{\sigma}}|\sZ)$, yielding an expected generalization guarantee in terms of $\text{eCMI}_{D}(\ell (A))$.

\subsection{Bounding eCMI from VC-dimension}
VC-dimension \cite{vapnik1974theory} of a concept class $\mathcal{H}$ is the largest number of data points $x_1,\dots,x_d \in \mathcal{X}$ such that all $2^d$ possible labelings can be realized by the hypotheses in the class.
In the following, we derive a guarantee on eCMI of the ERM learner in terms of VC-dimension.

According to the Sauer-Shelah lemma, the number of possible labelings of a supersample $\sz$ induced by the hypotheses in a concept class can be bounded by $\sum_{k=0}^d \binom{2n}{k} \leq e^2 \cdot n^d$, yielding a bound on eCMI in terms of the VC-dimension of the concept class:
$\text{eCMI}(\ell (A)) \leq \log |\mathcal{H}_{\sz}| \leq \log (e^2 \cdot n^d) \leq d\log n +2,$ where $\mathcal{H}_{\sz}$ denotes the projection of the concept class $\mathcal{H}$ onto a fixed arbitrary supersample $\sz \in \mathcal{Z}^{2n}$ realizable by $\mathcal{H}$. Note that \cite{pmlr-v125-steinke20a} uses $\log$ to denote the natural logarithm.
Interestingly, the same $O(d\log n)$ guarantee also applies to CMI of a globally consistent ERM.

Steinke and Zakynthinou \cite{pmlr-v125-steinke20a} show that any concept class with VC-dimension $d$ admits an empirical risk minimizer with CMI of order $O(d\log n)$. 
In particular, it focuses on deterministic algorithms that satisfy a global consistency property. 
For instance, a globally consistent ERM algorithm breaks ties according to a predefined ordering of hypotheses, when there are multiple ERMs that achieve the same empirical error rate. 
Thus, the entropy of the output hypotheses of a globally consistent algorithm is equal to the entropy of its output predictions on the supersample, which implies that CMI of a global consistent ERM is equal to the eCMI of the ERM learner. 

However, the upper bound in terms of VC-dimension is a worst-case entropy guarantee for ERM learners, where the worst-case distribution is a uniform distribution over the label space of the supersample. 
Steinke and Zakynthinou \cite{pmlr-v125-steinke20a} conjecture that there exists a deterministic or randomized learning algorithm such that $\text{CMI}_{\mathcal{D}}(A) \leq c\cdot d$ for every distribution $\mathcal{D}$, where $c$ is a universal constant.

\vspace{-8pt}
\subsection{Our Contributions}
We present a randomized Majority-of-5 algorithm $A$ with base learner $\mathcal{B}$. Consider a fixed (but arbitrary) concept class $\mathcal{H}$ of functions $h: \mathcal{X}\rightarrow \{0,1\}$ of VC-dimension $d\geq 1$. Let $n \geq d$ denotes the number of training data points input to $A$.
\begin{theorem}\label{thm:eCMI(A)}
For any base learner $\mathcal{B}: (\mathcal{X} \times \{0,1\})^{m} \to \{0,1\}^{\mathcal{X}}$ having guarantee that $\forall \mathcal{D}$ realizable, for i.i.d. data set $S_m \sim \mathcal{D}^m$, the expected error rate $\E_{S_m\sim \mathcal{D}^m}[\text{er}_{\mathcal{D}}(\mathcal{B}(S_m))] \leq R(m)$. The worst-case eCMI of a randomized Majority-of-5 (Algorithm \ref{algo:m5m3}) of such base learners is $\text{eCMI}(\ell(A)) \leq 1920\cdot nR(\frac{n}{5})$. For every $z\in  (\mathcal{X} \times \{0,1\})^n$, the expected training error rate of $A$ with respect to the randomness $\boldsymbol{\beta}$ of the algorithm
$\E_{\boldsymbol{\beta}}\allowbreak[\ell(A(z), z)] \allowbreak\leq 5\cdot R(\frac{n}{5})$.
\end{theorem}

\begin{theorem}\label{thm:eCMI(maj-of-3)}
    There exists a randomized algorithm $A : (\mathcal{X} \times \{0,1\})^n \to \{0,1\}^{\mathcal{X}}$ with $\text{eCMI}(\ell(A)) = O(d)$, and for every $z\in  (\mathcal{X} \times \{0,1\})^n$, the expected training error rate of $A$ with respect to the randomness $\boldsymbol{\beta}$ of the algorithm $\E_{\boldsymbol{\beta}}\allowbreak[\ell(A(z), z)] \allowbreak= O(\frac{d}{n})$.
\end{theorem}

\noindent
We also study distributional eCMI and present a simpler algorithm, a Majority-of-3 base learners with optimal expected error rate, that achieves the optimal guarantee on $\text{eCMI}_{\mathcal{D}}(\ell(A))$ with a simpler analysis. We discuss this in Section \ref{distributional eCMI}.

\subsection{Notation}
\noindent
Consider a fixed but arbitrary supersample $\sz = \{(x_1,y_1),\dots,(x_{2n},y_{2n})\} \in \mathcal{Z}^{2n}$ that is realizable by $\mathcal{H}$. The set of unlabeled data examples is denoted as $\sz_x = \{x_1,\dots,x_{2n}\}$.
We use $\mathcal{Z}=\mathcal{X}\times\{0,1\}$ to denote the space of labeled examples and $\mathcal{H}\subseteq \{0,1\}^{\mathcal{X}}$ to denote a concept class.
We denote by $h^*\in \mathcal{H}$ the unknown target function consistent with $\sz$.
The majority vote function is defined as $\text{Majority}(h_1,\dots,h_t)(x) = \1[h_1(x) + \dots + h_t(x) \geq \frac{t}{2}]$ for any classifiers $h_1,\dots, h_t:\mathcal{X}\rightarrow \{0,1\}$.

For random variables $X, Y$ and $Z$,
the mutual information between $X$ and $Y$ is denoted by $I(X;Y)$ and the conditional mutual information given $Z$ is denoted by $I(X;Y|Z)$.
The conditional entropy of $X$ given $Y$ is denoted by $H(X|Y)$. 
The binary entropy function\footnote{Note that we take $\log$ to denote the logarithm base 2, by a change of base, we have $h_2(p) = h_e(p)/\ln2$, thus $h_2(p) \geq h_e(p)$ for $p\in [0,1]$, hence our results extend to the binary entropy function defined with natural logarithms.} is denoted by $h(p) = -p\log p - (1-p)\log(1-p)$ for $p \in [0,1]$, or equivalently by $H(\1[.])$, the Shannon entropy of a Bernoulli random variable.
We denote by $H(\1[A(\sz_{\boldsymbol{\sigma}},\boldsymbol{\beta})(x) \neq h^*(x)] | \boldsymbol{\beta})$ the conditional binary entropy of errors of the output prediction of algorithm $A$ running on $\sz_{\boldsymbol{\sigma}}$ at any $x \in \sz_x$.
The subadditivity of Shannon entropy states that $H(X,Y) \leq H(X) + H(Y)$, and 
the chain rule for mutual information gives $I(X, Y ; Z) = I(X; Z) + I(Y ; Z|X)$.

\vspace{-8pt}\section{Our Approach}\label{Our apporach}
{\vskip -2mm}
In this section, we present a randomized learning algorithm with an optimal guarantee on the worst-case eCMI and illustrate the high level ideas of the analysis of the algorithm.

\vspace{-8pt}
\subsection{A randomized Majority-of-5 learner}

Given $n$ training data points $z \in \mathcal{Z}^n$,
the algorithm first performs a uniform random permutation $\boldsymbol{\pi}$ over $n$ training data points. Then it partitions the training data points into 5 disjoint pieces $S^{\boldsymbol{\pi}}_1,\dots,S^{\boldsymbol{\pi}}_5$ consisting of $\frac{n}{5}$ training data points each,
and adjusts for duplicates in the sequence of indices $I^{\boldsymbol{\pi}}_k$ of each disjoint data set $S_k^{\boldsymbol{\pi}}$, based on the relative frequencies of distinct values in an i.i.d. sequence of indices $\boldsymbol{U}$ from $\text{Unif}(\{1,\dots,n\})$. 
%
By doing this, each training data set is conditionally i.i.d. given $\boldsymbol{\sigma}$ from the empirical distribution on $n$ training data points. Then it runs a base learner $\mathcal{B}$ on each duplicates adjusted training data set $S^{\boldsymbol{\beta}}_1,\dots,S^{\boldsymbol{\beta}}_5$ to obtain output classifiers $h^{\boldsymbol{\beta}}_1, \dots, h^{\boldsymbol{\beta}}_5$, where $h^{\boldsymbol{\beta}}_k = \mathcal{B}(S_k^{\boldsymbol{\beta}})$ for any $k\in[5]$.
Finally, it combines the classifiers with a Majority-of-5.

We denote by $A : (\mathcal{X} \times \{0,1\})^n \to \{0,1\}^{\mathcal{X}}$ a randomized Majority-of-5 algorithm with base learner $\mathcal{B}: (\mathcal{X} \times \{0,1\})^{m} \to \{0,1\}^{\mathcal{X}}$. And $\boldsymbol{\beta} = (\boldsymbol{\pi},\boldsymbol{U})$ denotes the randomness of the algorithm, involving a random permutation $\boldsymbol{\pi}$ and a duplicates adjustment with an i.i.d. sequence $\boldsymbol{U}$.
In the following, we present the pseudo-code of the algorithm $A$ in Algorithm \ref{algo:m5m3}.

\vspace{3pt}
\begin{algorithm}[H]
\KwIn{$n$ data points $z_{}\in\Z^{n}$. Base learner $\mathcal{B}$.}

\KwResult{Classifier $h^{\boldsymbol{\beta}}: \mathcal{X} \to \{0,1\}$.}

Let $\boldsymbol{\pi}$ be uniform random permutation of $\sett{1,\dots,n}$ independent of $\boldsymbol{\sigma}$.

Let $\boldsymbol{U} \gets \{ \boldsymbol{u}_1, \boldsymbol{u}_2, \dots, \boldsymbol{u}_{\frac{n}{5}}\} \simiid \Unif{\sett{1, \dots, n} }$ independent of $\boldsymbol{\pi}$, $\boldsymbol{\sigma}$.

Let $\boldsymbol{\beta}=(\boldsymbol{\pi},\boldsymbol{U})$ denote the randomness of the algorithm.

\For{$k=1,\dots,5$}{
    Let $I^{\boldsymbol{\pi},\boldsymbol{U}}_k = \textsc{DuplicatesAdjustment}(\boldsymbol{U}, \{\boldsymbol{\pi}((k-1)\frac{n}{5}+1),\dots, \boldsymbol{\pi}(k\frac{n}{5})\})$.

    Let $S^{\boldsymbol{\beta}}_k = z_{I^{\boldsymbol{\pi},\boldsymbol{U}}_k}$ denote the training data set of the $k$-th learner.

    Let $h^{\boldsymbol{\beta}}_k = \mathcal{B}(S^{\boldsymbol{\beta}}_k)$.
}
\Return $h^{\boldsymbol{\beta}} = $Majority($h^{\boldsymbol{\beta}}_1, h^{\boldsymbol{\beta}}_2, h^{\boldsymbol{\beta}}_3, h^{\boldsymbol{\beta}}_4, h^{\boldsymbol{\beta}}_5$).
\caption{\textsc{Majority-of-5}($z_{}, \mathcal{B}$)}\label{algo:m5m3}
\end{algorithm}
\vspace{6pt}
\noindent
To analyze the eCMI of the algorithm, we define $\boldsymbol{\sigma}$ to be independent of $\boldsymbol{\beta} = (\boldsymbol{\pi},\boldsymbol{U})$. Moreover, we define the input to the algorithm as the $n$ training data points of the supersample,i.e.,$\sz_{\boldsymbol{\sigma}}$.
And we find it useful to let $\Gamma_k^{\boldsymbol{\pi}} = \{x_{\boldsymbol{\pi}((k-1)\frac{n}{5}+1))},\dots, x_{\boldsymbol{\pi}(k\frac{n}{5})} ,\dots, x_{n+\boldsymbol{\pi}((k-1)\frac{n}{5}+1))}, \dots, x_{n+\boldsymbol{\pi}(k\frac{n}{5})} \}$ denote the $k$-th part of the partition of $\sz_x$ consisting of $\frac{n}{5}$ pairs of unlabeled data points.
And we denote by $S^{\boldsymbol{\pi},\boldsymbol{\sigma}}_k = \{(x_{\boldsymbol{\sigma}({\boldsymbol{\pi}((k-1)\frac{n}{5}+1)})}, y_{\boldsymbol{\sigma}({\boldsymbol{\pi}((k-1)\frac{n}{5}+1)})}),\allowbreak \dots, \allowbreak (x_{\boldsymbol{\sigma}(\boldsymbol{\pi}(k\frac{n}{5}))}, y_{\boldsymbol{\sigma}(\boldsymbol{\pi}(k\frac{n}{5}))})\}$ the $k$-th disjoint subset of a uniform random permutation $\boldsymbol{\pi}$ of the $n$ training data points.
For every $k\in [5]$, we let $S_k^{\boldsymbol{\beta},\boldsymbol{\sigma}}$ denote the training data set of the $k$-th learner after adjusting for duplicates in $S_k^{\boldsymbol{\pi},\boldsymbol{\sigma}}$,i.e., $S^{\boldsymbol{\beta},\boldsymbol{\sigma}}_k = \sz_{\boldsymbol{\sigma}(I^{\boldsymbol{\pi},\boldsymbol{U}}_k)}$. And the output classifier of the $k$-th learner is $h^{\boldsymbol{\beta},\boldsymbol{\sigma}}_k = \mathcal{B}(S^{\boldsymbol{\beta},\boldsymbol{\sigma}}_k)$.

\vspace{-8pt}
\subsection{Proof Overview}
In this section, we give an overview of our approach.

Our approach starts with an analysis of the eCMI of a randomized learning algorithm, assuming that its internal randomness $\boldsymbol{\beta}$ is independent of the random swaps between training and ghost data points. 
Let $\sz \in \mathcal{Z}^{2n}$ be a fixed but arbitrary supersample that is realizable by $\mathcal{H}$.
We show that the eCMI of such randomized learning algorithm can be bounded by the sum of conditional entropies of errors at $2n$ data points of $\sz$. That is, 
$
\text{eCMI}(\ell(A)) \leq \sum_{i\leq 2n} H\big(\ell\big(A(\sz_{\boldsymbol{\sigma}},\boldsymbol{\beta}),\sz_{i}\big)\mid \boldsymbol{\beta}\big),
$
which is determined by the conditional probability of error over $\boldsymbol{\sigma}$ given $\boldsymbol{\beta}$ of the algorithm at $2n$ data points. 

Taking inspiration from the elements of proof of the Majority-of-3 learner in \cite{pmlr-v247-aden-ali24a} that achieves the optimal in-expectation generalization guarantee,
we define the $j$-th conditional error region of the $k$-th base learner as $R^{k,\boldsymbol{\beta}}_j \subseteq \sz_x \backslash \Gamma_k^{\boldsymbol{\pi}}$ based on the conditional probability of error of the base learner over $\boldsymbol{\sigma}$ given $\boldsymbol{\beta}$: $p^{\boldsymbol{\beta}}_k(x) = \mathbb{P}_{\boldsymbol{\sigma}|\boldsymbol{\beta}}(\mathcal{B}(S_k^{\boldsymbol{\beta},\boldsymbol{\sigma}})(x) \neq h^*(x)|\boldsymbol{\beta})$ for any $x\in \sz_x \backslash \Gamma_k^{\boldsymbol{\pi}}$.
Specifically, $R^{k,\boldsymbol{\beta}}_j$ contains all $x \in \sz_x \backslash \Gamma_k^{\boldsymbol{\pi}} $ such that $p^{\boldsymbol{\beta}}_k(x)\in(2^{-j},2^{-j+1}]$ for any $j\in \mathbb{N}$ and for any $k\in [5]$. 

In the following, we show that the constructed randomized Majority-of-5 learner outputs five base classifiers that are conditionally independent given $\boldsymbol{\beta}$ and the expected size of the conditional error region of each base learner $\mathbb{E}_{\boldsymbol{\beta}}[|R_{j}^{k,\boldsymbol{\beta}}|] \leq 2^j\cdot 2nR(\frac{n}{5})$ for any $j\in \mathbb{N}$ and for every $k\in [5]$. 
%
Moreover, if the base learner has an optimal expected error rate of order $O(\frac{d}{n})$, then the worst-case eCMI of this randomized learning algorithm is minimal, of order $O(d)$.
Next, we introduce the main ingredients of the construction of the algorithm
that achieves an optimal guarantee on the worst-case eCMI.
\vspace{-8pt}
\paragraph{Conditional independence of $S^{\boldsymbol{\beta},\boldsymbol{\sigma}}_1, S^{\boldsymbol{\beta},\boldsymbol{\sigma}}_2,\dots, S^{\boldsymbol{\beta},\boldsymbol{\sigma}}_{5}$ given $\boldsymbol{\beta}$.}
We partition the $n$ training data points into disjoint pieces such that the uniform random selection of training data sets $S^{\boldsymbol{\beta},\boldsymbol{\sigma}}_1, S^{\boldsymbol{\beta},\boldsymbol{\sigma}}_2,\dots, S^{\boldsymbol{\beta},\boldsymbol{\sigma}}_{5}$, induced by random swaps between training and ghost data points in disjoint subsets of the partition of $\sz$, are conditionally independent given $\boldsymbol{\beta}$. 
Thus, the output classifiers $h^{\boldsymbol{\beta},\boldsymbol{\sigma}}_1, \dots, h^{\boldsymbol{\beta},\boldsymbol{\sigma}}_{5}$ of the base learners trained on disjoint training data sets are conditionally independent given $\boldsymbol{\beta}$. The conditional probability over $\boldsymbol{\sigma}$ given $\boldsymbol{\beta}$ that two base learners indexed by $k,k'\in [5]$ simultaneously err at any $x\in \sz_x\backslash \Gamma_k^{\boldsymbol{\pi}}\cup\Gamma_{k'}^{\boldsymbol{\pi}} $ is therefore $p^{\boldsymbol{\beta}}_k(x)p^{\boldsymbol{\beta}}_{k'}(x)$.

\vspace{-8pt}
\paragraph{Conditionally i.i.d. training data set given $\boldsymbol{\sigma}$.} 
Since the $2n$ data points in a fixed supersample $\sz$ are arbitrary, the data points in any disjoint  subset of the partition of $\sz$ are generally non-i.i.d.
To obtain an expected error guarantee for each base learner running in disjoint subsets of the partition, we perform a uniform random permutation $\boldsymbol{\pi}$ over $n$ training data points such that each base learner observes a random sample from the $2n$ data points, with the aim that applying the standard concentration argument can yield a guarantee on the expected number of errors on data points in $\sz$, for each base learner.

However, deriving such a guarantee under sampling without replacement is technically challenging, especially for majority vote algorithms.
Nevertheless, we can reduce the problem to a setting with independent draws by introducing duplicate observations to a random sample drawn without replacement using algorithm \textsc{DuplicatesAdjustment} (Algorithm \ref{algo:dup_adjust}), thereby allowing us to invoke an i.i.d. expected error guarantee.
The algorithm samples an i.i.d. sequence $\boldsymbol{U}$ from $\Unif{\{1,\dots,n \}}$ and replaces its $\boldsymbol{N}$ distinct values with the first $\boldsymbol{N}$ distinct values of a sequence of indices $I^{\boldsymbol{\pi}}_k$ of a uniform random sample $S^{\boldsymbol{\pi},\boldsymbol{\sigma}}_k$ drawn without replacement from the $n$ training data points $\sz_{\boldsymbol{\sigma}}$.
The reconstructed sequence $I^{\boldsymbol{\pi},\boldsymbol{U}}_k $ can be proved to be an i.i.d. sequence from $\Unif{\sett{1, \dots, n} }$.
Since an empirical distribution is a discrete uniform distribution over the $n$ distinct indices, the training data set of each base learner $S^{\boldsymbol{\beta},\boldsymbol{\sigma}}_k = \sz_{\boldsymbol{\sigma}(I^{\boldsymbol{\pi},\boldsymbol{U}}_k)}$ is conditionally i.i.d. from $\Unif{\sett{x_{\boldsymbol{\sigma}(1)},\dots,x_{\boldsymbol{\sigma}(n)}}}$ given $ \boldsymbol{\sigma}$, for any $k\in[5]$.

\paragraph{Expected number of errors and expected size of the conditional error region.}
The challenge in establishing a guarantee on the expected number of errors for a base learner on data points in a fixed arbitrary supersample $\sz$ originates from the dependence between the training and ghost data points under random swaps and random permutation. Conditioning on part of the randomness enables us to exploit the resulting conditional independence.

Our key observation is that conditional on $\boldsymbol{\beta}$, the uniform random choice of the test data point $\{x_{\boldsymbol{\pi}(i)}, x_{\boldsymbol{\pi}(i)+n}\} \notin \Gamma_k^{\boldsymbol{\pi}}$ induced by $\boldsymbol{\sigma}(\boldsymbol{\pi}(i))$, is independent of the uniform random selection of the training data set $S^{\boldsymbol{\beta},\boldsymbol{\sigma}}_k$ induced by the swap variables
in the $k$-th part of the partition of $\sz$, for any $k\in [5]$. 
Thus, for the $k$-th base learner, the expected number of errors over the random swaps within the $k$-th part of the partition, among the other $\frac{4n}{5}$ pairs of data points 
is equivalent to two times the expected number of errors over the randomness of all $n$ swap variables in $\boldsymbol{\sigma}$, 
among the other $\frac{4n}{5}$ training data variables $x_{\boldsymbol{\sigma}(\boldsymbol{\pi}(i))} \notin \Gamma_k^{\boldsymbol{\pi}}$.
Hence it suffices to count the expected number of errors of the base learner over $\boldsymbol{\beta},\boldsymbol{\sigma}$ among the other $\frac{4n}{5}$ training data variables.

Recall that we run \textsc{DuplicatesAdjustment} (Algorithm \ref{algo:dup_adjust}) on the $n$ training 
data points to emulate sampling a conditionally i.i.d. training data set $S_k^{\boldsymbol{\beta},\boldsymbol{\sigma}}$ from $\Unif{\sett{x_{\boldsymbol{\sigma}(1)},\dots,x_{\boldsymbol{\sigma}(n)}}}$ given $\boldsymbol{\sigma}$.
We can therefore condition on $\boldsymbol{\sigma}$ and the standard concentration argument implies that 
the expected number of errors over $\boldsymbol{\beta}$ of the base learner among $n$ training data points
equals
the i.i.d. expected error guarantee, which is at most $nR(\frac{n}{5})$.
Since the number of errors among the other $\frac{4n}{5}$ training data points is at most that among $n$ training data points, to avoid complications due to dependencies among pairs of data points in $\Gamma^{\boldsymbol{\pi}}_k$ induced by random swaps (e.g. data points are paired in some adversarial way),
we have each base learner admit a guarantee on the expected number of errors over $\boldsymbol{\beta},\boldsymbol{\sigma}$ among the other $\frac{4n}{5}$ pairs of data points only, which is at most two times the i.i.d. guarantee on the expected number of errors of the base learner among $n$ training data points and hence at most $2nR(\frac{n}{5})$.

By the linearity of conditional expectation, the expected number of errors on $\sz_x \setminus \Gamma_k^{\boldsymbol{\pi}}$ over $\boldsymbol{\beta}, \boldsymbol{\sigma}$ equals the expectation over $\boldsymbol{\beta}$ of the sum of conditional probabilities of error over $\boldsymbol{\sigma}$ given $\boldsymbol{\beta}$ on $\sz_x \setminus \Gamma_k^{\boldsymbol{\pi}}$. With an application of Markov's inequality, we have the expected size of the $j$-th conditional error region over $\boldsymbol{\beta}$: $\mathbb{E}_{\boldsymbol{\beta}}[|R_{j}^{k,\boldsymbol{\beta}}|] \leq 2^j\cdot 2nR(\frac{n}{5})$ where $R_{j}^{k,\boldsymbol{\beta}} \subseteq \sz_x \backslash \Gamma_k^{\boldsymbol{\pi}}$ for any $j\in \mathbb{N}$.

\paragraph{Majority-of-5 base learners.}
From the above analysis, we know that the $k$-th base learner, running in the $k$-th part of the partition, might not have a guarantee on the conditional probability of error $p^{\boldsymbol{\beta}}_{k}(x)$ at data points $x\in \Gamma_k^{\boldsymbol{\pi}}$ in the $k$-th part of the partition, while it
can guarantee $p^{\boldsymbol{\beta}}_{k}(x)$ at the data points $x\notin \Gamma_k^{\boldsymbol{\pi}}$ in the other parts of the partition. 
And the base classifiers $h^{\boldsymbol{\beta},\boldsymbol{\sigma}}_1, \dots, h^{\boldsymbol{\beta},\boldsymbol{\sigma}}_5$ are conditionally independent given $\boldsymbol{\beta}$. 
Thus, for any data point in $k$-th part of the partition, it takes at least two of the other base classifiers to err when the majority vote classifier errs to get the benefit of the product of conditional probability of error over $\boldsymbol{\sigma}$ given $\boldsymbol{\beta}$.

To see why we use five base learners rather than three, we consider a worst-case scenario in which the $k$-th base learner always errs at $x$, for any data point $x\in \Gamma_k^{\boldsymbol{\pi}}$ in the $k$-th part of the partition.
Since the majority vote classifier errs when at least half of the base classifiers err,
if we have 3 base learners, then the majority vote classifier errs only if at least one of the other two base classifiers err. 
In contrast, with 5 base learners, at least two of the other 4 base classifiers must err when the majority vote classifier errs. 
Thus, the conditional probability of error of such Majority-of-5 learner can be bounded by the probability that at least two of the other 4 base learners err.
By the conditional independence of the base classifiers,
the conditional probability of error of the majority vote learner at $x$ can be bounded by a squared conditional probability of error of the base learner at $x$, which is sufficient to compensate for the exponential dependence on $j$ in the bound on the expected size of the conditional error region: $2^j\cdot 2n R(\frac{n}{5})$, yielding a convergent series of conditional binary entropies of errors at data points in the $k$-th part of the partition, for any $k\in [5]$.

\vspace{-6pt}
\paragraph{Optimal guarantee on eCMI.}
What remains to show is the guarantee on the worst-case eCMI of the constructed randomized Majority-of-5 learner. 
Without loss of generality, consider any data point $x\in \Gamma_1^{\boldsymbol{\pi}}$ in the first part of the partition of $\sz$. 
From the above analysis, we know that  
it suffices to bound the probability over random swaps that such Majority-of-5 learner errs at $x$ by the probability that there exists at least two of the other 4 base learners indexed by $k,k'\geq 2$ both err at $x$: 
$\mathbb{P}_{\boldsymbol{\sigma}|\boldsymbol{\beta}}(A(\sz_{\boldsymbol{\sigma}},\boldsymbol{\beta})(x) \neq h^*(x)|\boldsymbol{\beta}) \leq \binom{4}{2}\cdot p^{\boldsymbol{\beta}}_k(x)p^{\boldsymbol{\beta}}_{k'}(x)$.

Since the binary entropy function is bounded by a monotone non-decreasing function $f(p) = -p\log p + 2p = p\log \frac{4}{p}$ on $[0,1]$, combined with the inequality $p^{\boldsymbol{\beta}}_k(x)p^{\boldsymbol{\beta}}_{k'}(x) \leq \max\{p^{\boldsymbol{\beta}}_k(x)^2,p^{\boldsymbol{\beta}}_{k'}(x)^2\}$, we have 
the conditional entropy of errors on data points in $\Gamma_k^{\boldsymbol{\pi}}$ bounded by $2\cdot\binom{4}{2}\cdot\sum_{j=1}^{\infty}\E_{\boldsymbol{\beta}}[|R^{k,\boldsymbol{\beta}}_j|]\cdot 2^{-2j+2} \cdot 2j \leq 384 \cdot nR(\frac{n}{5})$.
By the symmetry of a uniform random permutation over $n$ training data points, the conditional entropy of errors of the randomized majority vote learner on $2n$ data points is five times that on the data points in the first part of the partition $\Gamma_1^{\boldsymbol{\pi}}$.
Hence, we have that the conditional entropy of errors given $\boldsymbol{\beta}$ on $2n$ data points is at most $1920 \cdot nR(\frac{n}{5})$, which implies that the worst-case eCMI of the randomized Majority-of-5 learner is of order $O(nR(\frac{n}{5}))$.
Specifically, if the base learner has an optimal expected error rate of order $O(\frac{d}{n})$, 
then the worst-case eCMI of the randomized Majority-of-5 learner is of order $O(d)$.

\vspace{-6pt}
\section{Reduction to Conditional entropy}
We start with the following lemma that 
the eCMI of a randomized learning algorithm can be bounded by the sum of conditional entropies of losses on a fixed arbitrary supersample.
\begin{lemma}\label{lemma:epy_of_points}
For any supersample $\sz \in \mathcal{Z}^{2n}$. We have 
$$I\big(\vec{\ell}\big(A(\sz_{\boldsymbol{\sigma}},\boldsymbol{\beta}),\sz\big);\boldsymbol{\sigma}\big) \leq H\big(\vec{\ell}\big(A(\sz_{\boldsymbol{\sigma}},\boldsymbol{\beta}),\sz\big)\mid \boldsymbol{\beta}\big) \leq \sum_{i\leq 2n} H\big(\ell\big(A(\sz_{\boldsymbol{\sigma}},\boldsymbol{\beta}),\sz_{i}\big)\mid \boldsymbol{\beta}\big).$$
\end{lemma}
\begin{proof}
Consider a fixed but arbitrary $\sz \in \mathcal{Z}^{2n}$. According to the chain rule for mutual information, we have
\begin{align*}
    I\big(\vec{\ell}\big(A(\sz_{\boldsymbol{\sigma}},\boldsymbol{\beta}),\sz\big);\boldsymbol{\sigma}\big) &= I\big(\big(\vec{\ell}\big(A(\sz_{\boldsymbol{\sigma}},\boldsymbol{\beta}),\sz\big), \boldsymbol{\beta} \big) ; \boldsymbol{\sigma}\big) - I\big(\boldsymbol{\beta}; \boldsymbol{\sigma} \mid \vec{\ell}\big(A(\sz_{\boldsymbol{\sigma}},\boldsymbol{\beta}),\sz\big)\big).
\end{align*}
By the non-negativity of mutual information, the above expression is bounded by
$$
   I\big(\big(\vec{\ell}\big(A(\sz_{\boldsymbol{\sigma}},\boldsymbol{\beta}),\sz\big), \boldsymbol{\beta} \big) ; \boldsymbol{\sigma}\big).
$$
Applying the chain rule for Shannon entropy, we have the above expression equals
\begin{align*}
    &H\big(\vec{\ell}\big(A(\sz_{\boldsymbol{\sigma}},\boldsymbol{\beta}),\sz\big)\mid\boldsymbol{\beta}\big) + H\big(\boldsymbol{\beta}\big) - H\big(\vec{\ell}\big(A(\sz_{\boldsymbol{\sigma}},\boldsymbol{\beta}),\sz\big) \mid \boldsymbol{\beta}, \boldsymbol{\sigma} \big) - H\big(\boldsymbol{\beta} \mid \boldsymbol{\sigma}\big).
\end{align*}
Since $H\big(\vec{\ell}\big(A(\sz_{\boldsymbol{\sigma}},\boldsymbol{\beta}),\sz\big) \mid \boldsymbol{\beta}, \boldsymbol{\sigma} \big) = 0$ and $H(\boldsymbol{\beta}) = H(\boldsymbol{\beta}\mid \boldsymbol{\sigma})$, 
we have the above expression equals
\begin{align*}
&H\big(\vec{\ell}\big(A(\sz_{\boldsymbol{\sigma}},\boldsymbol{\beta}),\sz\big)\mid\boldsymbol{\beta}\big).
\end{align*}
By the subadditivity of Shannon entropy, we conclude
\begin{align*}
H\big(\vec{\ell}\big(A(\sz_{\boldsymbol{\sigma}},\boldsymbol{\beta}),\sz\big)\mid\boldsymbol{\beta}\big) \leq \sum_{i\leq 2n} H\big(\ell\big(A(\sz_{\boldsymbol{\sigma}},\boldsymbol{\beta}),\sz_{i}\big)\mid\boldsymbol{\beta}\big),
\end{align*}
as claimed.
\end{proof}

\section{Guarantee on the expected size of the conditional error region}
In this section, we introduce the \textsc{DuplicatesAdjustment} algorithm that constructs a conditionally i.i.d. data set $S_k^{\boldsymbol{\beta},\boldsymbol{\sigma}}$ given $\boldsymbol{\sigma}$ from the empirical distribution of $n$ training data points. 

The developed algorithm forms the basis of the proof of Lemma \ref{lemma:expected_error}, which derives a guarantee for each base learner on the expected number of errors over $\boldsymbol{\beta},\boldsymbol{\sigma}$ among the other $\frac{4n}{5}$ pairs of data points in a fixed arbitrary realizable supersample $\sz \in \mathcal{Z}^{2n}$.

In Lemma \ref{expected_err_region}, we prove that such expected error guarantee of the base learner yields a guarantee on the expected size of its conditional error regions, which is one of the ingredients for a more elaborate analysis on the conditional entropy of errors of the randomized Majority-of-5 learner in the next section.

\subsection{Duplicates adjustment}
The algorithm adjusts for duplicates in a uniform random sample $S^{\boldsymbol{\pi},\boldsymbol{\sigma}}_k$ drawn without replacement from the $n$ training data points $\sz_{\boldsymbol{\sigma}}$ by 
sampling an i.i.d. sequence $\boldsymbol{U}$ from $\Unif{\{1,\dots,n \}}$ and replacing its $N$ distinct values with the first $N$ distinct values of the sequence of indices $I^{\boldsymbol{\pi}}_k$ of $S^{\boldsymbol{\pi},\boldsymbol{\sigma}}_k$.
Since an empirical distribution on $n$ training data points is a discrete uniform distribution over the $n$ distinct indices, if we can construct an i.i.d. index sequence $I^{\boldsymbol{\pi},\boldsymbol{U}}_k$
from $\Unif{\sett{1, \dots, n} }$, then $S_k^{\boldsymbol{\beta},\boldsymbol{\sigma}} = \sz_{\boldsymbol{\sigma}(I^{\boldsymbol{\pi},\boldsymbol{U}}_k)}$  is a conditionally i.i.d. data set from $\Unif{\sett{x_{\boldsymbol{\sigma}(1)},\dots,x_{\boldsymbol{\sigma}(n)}}}$ given $\boldsymbol{\sigma}$.

The motivation for this algorithm is to convert a uniform random sample drawn without replacement into an i.i.d. sample such that each base learner, running on a disjoint subset of the partition of $\sz$, can have an i.i.d. guarantee on the expected number of errors on all realizable supersamples $\sz \in \mathcal{Z}^{2n}$, thus yielding a guarantee on the worst-case eCMI for the randomized majority vote learner.
We present the pseudo-code of the algorithm, shown in Algorithm \ref{algo:dup_adjust}.\\

\begin{algorithm}[H]
\KwIn{$\{u_1,\dots,u_{\frac{n}{5}}\}$, $\{k_1,\dots,k_{\frac{n}{5}}\}$.}
\KwResult{$\{i_1,\dots,i_{\frac{n}{5}}\}$.}
Let $N$ be the number of distinct values in $\{u_1, u_2, \dots, u_{\frac{n}{5}}\}$.

Let $v_1,\dots,v_N$ be these distinct values (in order of first appearance).

\For{$t=1,\dots,\frac{n}{5}$}{
    \For{$s = 1, \dots, N$}{
        \If{$u_t == v_s$}{
            define $i_{t}=k_s$.
        }
    }
}
\Return{$\{i_1,\dots,i_{\frac{n}{5}}\}$}.
\caption{\textsc{DuplicatesAdjustment}($\{u_1,\dots,u_{\frac{n}{5}}\}$, $\{k_1,\dots,k_{\frac{n}{5}}\}$)}\label{algo:dup_adjust}
\end{algorithm}
\paragraph{Brief Overview.}
Before diving into the formal proof, we introduce some notations and briefly discuss some steps of Algorithm \ref{algo:dup_adjust} along with some high level ideas in the analysis.
Let $\boldsymbol{U}= ( \boldsymbol{u}_1, \boldsymbol{u}_2, \dots, \boldsymbol{u}_{\frac{n}{5}}) \simiid \Unif{\sett{1, \dots, n} }$ denote an i.i.d. sequence of indices.
Let $\boldsymbol{N}$ be the number of distinct indices in $\boldsymbol{U}$.
Let $(\boldsymbol{v}_1,\dots,\boldsymbol{v}_{\boldsymbol{N}})$ be the sequence of distinct values in $\boldsymbol{U}$. 
Let $I^{\boldsymbol{\pi}}_k = (\boldsymbol{\pi}((k-1)\frac{n}{5}+1),\dots, \boldsymbol{\pi}(k\frac{n}{5}))$
denote the sequence of indices of a uniform random sample $S^{\boldsymbol{\pi},\boldsymbol{\sigma}}_k$ drawn without replacement from the $n$ training data points $\sz_{\boldsymbol{\sigma}}$.
Let $I^{\boldsymbol{\pi},\boldsymbol{U}}_k $ denote the sequence of indices after the duplicates adjustment of the distinct values in $I^{\boldsymbol{\pi}}_k$, based on the relative frequencies of the distinct values in an i.i.d. sequence $\boldsymbol{U} \simiid \Unif{\{1,…,n\}}$.

We will use $\boldsymbol{U}$ to adjust the duplicates of distinct values in $I^{\boldsymbol{\pi}}_k$ to produce the sequence of indices $I^{\boldsymbol{\pi},\boldsymbol{U}}_k$. Concretely, the for-loop of Algorithm \ref{algo:dup_adjust} in steps 3-9 replaces the $s$th distinct value $v_s$ in an i.i.d. index sequence ${U}$ with the $s$th distinct value $k_s$ in the uniform random index sequence $I^{\boldsymbol{\pi}}_k$ drawn without replacement from $\{1,\dots,n\}$, for every $s\in [N]$. At a high level, the algorithm replaces the $N$ distinct values $(v_1,\dots,v_N)$ in $U$ with the first $N$ distinct values of $I_k^{\boldsymbol{\pi}}$. 

The intuition behind the algorithm builds on the observation that a sample is uniquely characterized by its empirical distribution and the ordering of the sample values. Equivalently, it is determined by its distinct values, a histogram recording the relative frequencies of the distinct values and the ordering of the sample values. 
More formally, for any i.i.d. sequence $(\boldsymbol{u}_1, \boldsymbol{u}_2, \dots, \boldsymbol{u}_{\frac{n}{5}}) \in \{1,\dots,n\}^{\frac{n}{5}}$,
we define $j_i^{\boldsymbol{u}}$ as the index $j$ such that 
$\boldsymbol{u}_i = \boldsymbol{v}_j$ for all $i\leq \frac{n}{5}$. 
Then $G(\boldsymbol{U}) = (j_1^{\boldsymbol{u}},\dots,j_{\frac{n}{5}}^{\boldsymbol{u}})$ is an index sequence that consists of the indices of the distinct values that appears at each position in $\boldsymbol{U}$, thereby encoding both the histogram of relative frequencies and the ordering of the values in $\boldsymbol{U}$. 
For any sequence of distinct values $(w_1, \dots, w_{\boldsymbol{N}})$,
we define a function $f (G(\boldsymbol{U}), w_1, \dots, w_{\boldsymbol{N}})=(w_{j_1^{\boldsymbol{u}}},\dots,w_{j_{\frac{n}{5}}^{\boldsymbol{u}}})$ that constructs a sequence whose distinct values are $(w_1,\dots,w_{\boldsymbol{N}})$, with its histogram of relative frequencies and the ordering of the values specified by $G(\boldsymbol{U})$.
In particular, we have $\boldsymbol{U} = f(G(\boldsymbol{U}), \boldsymbol{v}_1, \dots,\boldsymbol{v}_{\boldsymbol{N}})$ and $I^{\boldsymbol{\pi},\boldsymbol{U}}_k = f(G(\boldsymbol{U}), \boldsymbol{\pi}((k-1)\frac{n}{5}+1),\dots, \boldsymbol{\pi}((k-1)\frac{n}{5}+ \boldsymbol{N}))$.

In the following lemma, we prove that 
the reconstructed sequence of indices $I^{\boldsymbol{\pi},\boldsymbol{U}}_k = (i_1,\dots,i_{\frac{n}{5}})$ has the same distribution as an i.i.d. sequence drawn from $\Unif{\sett{1, \dots, n} }$, implying that $S_k^{\boldsymbol{\beta},\boldsymbol{\sigma}}=\sz_{\boldsymbol{\sigma}(I^{\boldsymbol{\pi},\boldsymbol{U}}_k)}$  is a conditionally i.i.d. data set from $\Unif{\sett{x_{\boldsymbol{\sigma}(1)},\dots,x_{\boldsymbol{\sigma}(n)}}}$ given $\boldsymbol{\sigma}$.

\begin{lemma}\label{lemma:duplicates_adjust}
By adjusting for duplicates in a uniform random sample $S^{\boldsymbol{\pi},\boldsymbol{\sigma}}_k$ of size $\frac{n}{5}$ drawn without replacement from the $n$ training data points, the training data set $S_k^{\boldsymbol{\beta},\boldsymbol{\sigma}}$ of size $\frac{n}{5}$ is conditionally i.i.d. from the empirical distribution of $n$ training data points given $\boldsymbol{\sigma}$. That is, for any $k\in[5]$,
\begin{align*}
        S_k^{\boldsymbol{\beta},\boldsymbol{\sigma}} | \boldsymbol{\sigma} \simiid \Unif{\sett{x_{\boldsymbol{\sigma}(1)},\dots,x_{\boldsymbol{\sigma}(n)}}} | \boldsymbol{\sigma}.
    \end{align*}
\end{lemma}
\begin{proof}

Since $S_1^{\boldsymbol{\beta},\boldsymbol{\sigma}},\dots,S_5^{\boldsymbol{\beta},\boldsymbol{\sigma}}$ has the same conditional distribution given $\boldsymbol{\sigma}$, without loss of generality, we consider $k=1$.
To show that the training data set $S_1^{\boldsymbol{\beta},\boldsymbol{\sigma}}$ is conditionally i.i.d. given $\boldsymbol{\sigma}$ from the empirical distribution of $n$ training data points, it suffices to show that the sequence of indices $I^{\boldsymbol{\pi},\boldsymbol{U}}_1$  of the training data set is i.i.d. $\Unif{\sett{1, \dots, n}}$. 

Recall that the algorithm adjusts for duplicates in $I^{\boldsymbol{\pi}}_1$
by replacing the $\boldsymbol{N}$ distinct values $(\boldsymbol{v}_1,\dots,\boldsymbol{v}_{\boldsymbol{N}})$ in an i.i.d. sequence $\boldsymbol{U}$ with the first $\boldsymbol{N}$ distinct values $(\boldsymbol{\pi}(1),\dots,\boldsymbol{\pi}(\boldsymbol{N}))$ in $I^{\boldsymbol{\pi}}_1$.
To show that the reconstructed sequence $I_1^{\boldsymbol{\pi},\boldsymbol{U}}$ is i.i.d. $\Unif{\{1,\dots,n\}}$, we first characterize the distribution of $I_1^{\boldsymbol{\pi},\boldsymbol{U}}$.
Indeed, $I_1^{\boldsymbol{\pi},\boldsymbol{U}}$ is constructed by first sampling a random histogram of relative frequencies of distinct values and the ordering of the values from the conditional distribution of $G(\boldsymbol{U})$ given the number of the distinct values $\boldsymbol{N}$. Then, conditional on $\boldsymbol{N}$, a sequence of $\boldsymbol{N}$ distinct values $(\boldsymbol{\pi}(1),\dots,\boldsymbol{\pi}(\boldsymbol{N}))$ is drawn uniformly at random without replacement from $\{1,\dots,n\}$, independently of $G(\boldsymbol{U})$.
Assigning these distinct values to $G(\boldsymbol U)$ yields the reconstructed sequence $I_1^{\boldsymbol{\pi},\boldsymbol U}$.

By the chain rule of probability, the above sampling process can be formalized as sampling sequentially from the following sequence of conditional distributions.
Therefore,
the reconstructed sequence of indices $I_1^{\boldsymbol{\pi},\boldsymbol{U}}$ is indeed sampled from the marginal distribution of the following joint distribution,
\begin{align*}
P_{I_1^{\boldsymbol{\pi},\boldsymbol{U}},G(\boldsymbol{U}),\boldsymbol{N}} = P_{I_1^{\boldsymbol{\pi},\boldsymbol{U}} | G(\boldsymbol{U}),\boldsymbol{N}} \cdot P_{G(\boldsymbol{U}) | \boldsymbol{N}} \cdot P_{\boldsymbol{N}}.
\end{align*}
To show that $I_1^{\boldsymbol{\pi},\boldsymbol{U}}$ has the same distribution as the i.i.d. sequence $\boldsymbol{U}$,
it suffices to establish $P_{I_1^{\boldsymbol{\pi},\boldsymbol{U}} | G(\boldsymbol{U}),\boldsymbol{N}}=P_{\boldsymbol{U}| G(\boldsymbol{U}),\boldsymbol{N}}$. 

Note that because $(\boldsymbol{\pi}(1),\dots,\boldsymbol{\pi}(\boldsymbol{N}))$ consists of the first $\boldsymbol{N}$ elements of a uniform random permutation, they form a uniform random sample drawn without replacement from $\{1,\dots,n\}$ given $\boldsymbol{N}$. 
Similarly, the sequence of distinct values $(\boldsymbol{v}_1, \dots,\boldsymbol{v}_{\boldsymbol{N}})$ of an i.i.d. sequence $\boldsymbol{U}$ also form a uniform random sample without replacement from $\sett{1,\dots,n}$ given $\boldsymbol{N}$. It follows that
\begin{equation}
P_{\boldsymbol{\pi}(1),\dots,\boldsymbol{\pi}(\boldsymbol{N})|\boldsymbol{N}}=P_{\boldsymbol{v}_1, \dots,\boldsymbol{v}_{\boldsymbol{N}}|\boldsymbol{N}}.
\label{eq:pi-v}
\end{equation}
Since $\boldsymbol{U}= ( \boldsymbol{u}_1, \boldsymbol{u}_2, \dots, \boldsymbol{u}_{\frac{n}{5}})$ is a sequence of i.i.d. $\Unif{\sett{1, \dots, n} }$ random variables, we know that $G(\boldsymbol{U})$ is conditionally independent of the uniform random choice of the $\boldsymbol{N}$ distinct values $(\boldsymbol{v}_1,\dots,\boldsymbol{v}_{\boldsymbol{N}})$ given $\boldsymbol{N}$. Moreover, by the construction of the algorithm, we also have that $G(\boldsymbol{U})$ is conditionally independent of the uniform random sequence of distinct values $(\boldsymbol{\pi}(1),\dots,\boldsymbol{\pi}(\boldsymbol{N}))$  drawn without replacement given $\boldsymbol{N}$. 
Therefore, \eqref{eq:pi-v} implies that
\begin{align}
\label{eq:cond_indep}
P_{\boldsymbol{\pi}(1),\dots,\boldsymbol{\pi}(\boldsymbol{N})|G(\boldsymbol{U}),\boldsymbol{N}} = P_{\boldsymbol{v}_1, \dots,\boldsymbol{v}_{\boldsymbol{N}}|G(\boldsymbol{U}),\boldsymbol{N}}.
\end{align}
Note that by definition, both $I^{\boldsymbol{\pi},\boldsymbol{U}}_1$ and $\boldsymbol{U}$ are obtained by applying the same function $f$ to $(\boldsymbol{\pi}(1),\dots,\boldsymbol{\pi}(\boldsymbol{N}))$ and $(\boldsymbol{v}_1, \dots,\boldsymbol{v}_{\boldsymbol{N}})$ respectively.  Together with~\eqref{eq:cond_indep}, this immediately implies that
$$
P_{I^{\boldsymbol{\pi},\boldsymbol{U}}_1|G(\boldsymbol{U}),\boldsymbol{N}} = P_{\boldsymbol{U}|G(\boldsymbol{U}),\boldsymbol{N}}.
$$
Altogether, we can conclude that $(I_1^{\boldsymbol{\pi},\boldsymbol{U}},G(\boldsymbol{U}),\boldsymbol{N})$ and $(\boldsymbol{U},G(\boldsymbol{U}),\boldsymbol{N})$ have identical joint distributions, which implies that the reconstructed sequence of indices $I_{1}^{\boldsymbol{\pi, U}}$ has an identical marginal distribution as the i.i.d. sequence $\boldsymbol{U}$. 
%

Since an empirical distribution on $n$ training data points is equivalent to a uniform distribution over $n$ distinct indices,
we have that the training data set $S_1^{\boldsymbol{\beta},\boldsymbol{\sigma}} = \sz_{\boldsymbol{\sigma}(I^{\boldsymbol{\pi},\boldsymbol{U}}_1)} $ is conditionally i.i.d. from $\Unif{\sett{x_{\boldsymbol{\sigma}(1)},\dots,x_{\boldsymbol{\sigma}(n)}}}$ given $\boldsymbol{\sigma}$.
This concludes the proof of the lemma.
\end{proof}

\subsection{Expected number of errors}
In the following lemma, we show that, conditional on $\boldsymbol{\beta}$, for each base learner, its expected number of errors over the random swaps in the $k$-th part of the partition, among the other $\frac{4n}{5}$ pairs of data points is equivalent to two times the expected number of errors over random swaps of all $n$ swap variables among the other $\frac{4n}{5}$ training data variables. 

Since the training data set $S_k^{\boldsymbol{\beta},\boldsymbol{\sigma}}$ is conditionally i.i.d. given $\boldsymbol{\sigma}$ after duplicates adjustment, we can provide a guarantee on the expected number of errors of each base learner over $\boldsymbol{\beta}$ given $\boldsymbol{\sigma}$ among the other $\frac{4n}{5}$ pairs of data points, that is at most two times the i.i.d. guarantee on the expected number errors of the base learner among the $n$ training data points.
\newpage

\begin{lemma}\label{lemma:expected_error}
      The expected number of errors of the $k$-th base learner $\mathcal{B}$ over random swaps, random permutation and duplicates adjustment, among the other $\frac{4n}{5}$ pairs of data points that are not in $\Gamma_{k}^{\boldsymbol{\pi}}$, for any $k\in[5]$ is
     \[
     \E_{\boldsymbol{\beta}, \boldsymbol{\sigma}}\Bigg[\sum_{x\notin \Gamma_{k}^{\boldsymbol{\pi}}} \1 \left[\mathcal{B}\left(S_k^{\boldsymbol{\beta},\boldsymbol{\sigma}}\right)\left(x\right) \neq h^*\left(x\right)\right]\Bigg] \leq 2nR\left(\frac{n}{5}\right).
     \]
\end{lemma}
\begin{proof}
Fix any realizable $\sz \in \mathcal{Z}^{2n}$.
Without loss of generality, consider the 5th base learner with training data set $S^{\boldsymbol{\beta},\boldsymbol{\sigma}}_5$ in the 5th part of the partition of $\sz$.

We first condition on permutations and duplicates adjustments. Then we have that the expected number of errors of the base learner over random swaps in the 5th part of the partition, among the other $\frac{4n}{5}$ pairs of data points $\{x_{\boldsymbol{\pi}(i)}, x_{\boldsymbol{\pi}(i)+n}\} \notin \Gamma_5^{\boldsymbol{\pi}}$, for all $i\in [\frac{4n}{5}]$, is
\begin{align*}
    &\E_{\boldsymbol{\beta}}\Bigg[\E_{\boldsymbol{\sigma} | \boldsymbol{\beta}}\bigg[\sum_{x \notin \Gamma_5^{\boldsymbol{\pi}}} \1\bigg[\mathcal{B}\big(S_5^{\boldsymbol{\beta},\boldsymbol{\sigma}}\big)\left(x\right) \neq h^*\left(x\right) \bigg] \bigg| \boldsymbol{\beta}\bigg]\Bigg]\\
    =&\E_{\boldsymbol{\beta}}\left[\E_{\boldsymbol{\sigma} | \boldsymbol{\beta}}\Bigg[ \sum_{i=1}^{\frac{4n}{5}} \1 \left[\mathcal{B}\left(S_5^{\boldsymbol{\beta},\boldsymbol{\sigma}}\right)\left(x_{\boldsymbol{\pi}\left(i\right)}\right) \neq h^*\left(x_{\boldsymbol{\pi}\left(i\right)}\right)\right] + \1 \left[\mathcal{B}\left(S_5^{\boldsymbol{\beta},\boldsymbol{\sigma}}\right)\left(x_{\boldsymbol{\pi}(i)+n}\right) \neq h^*\left(x_{\boldsymbol{\pi}(i)+n}\right)\right] \bigg| \boldsymbol{\beta}\Bigg]\right].
\end{align*}
By the linearity of conditional expectation, we can focus on the sum of conditional probabilities of error of the base learner over $\boldsymbol{\sigma}$ given $\boldsymbol{\beta}$ at $\frac{4n}{5}$ pairs of test data points.
Since $x_{\boldsymbol{\pi}(i)}$ or $x_{\boldsymbol{\pi}(i)+n}$ is swapped to the training side with probability $\frac{1}{2}$, 
we have the above expression equals
\begin{multline*}
 \E_{\boldsymbol{\beta}}\Bigg[\sum_{i=1}^{\frac{4n}{5}} \E_{\boldsymbol{\sigma} | \boldsymbol{\beta}}\bigg[2\mathbb{P}_{\boldsymbol{\sigma}\left(\boldsymbol{\pi}\left(i\right)\right) | S^{\boldsymbol{\beta},\boldsymbol{\sigma}}_5,\boldsymbol{\beta}}\left(\boldsymbol{\sigma}\left(\boldsymbol{\pi}\left(i\right)\right) = \boldsymbol{\pi}(i) \big| S^{\boldsymbol{\beta},\boldsymbol{\sigma}}_5,\boldsymbol{\beta}\right) 
 \cdot \1 \left[ \mathcal{B}\left(S_5^{\boldsymbol{\beta},\boldsymbol{\sigma}}\right)\left(x_{\boldsymbol{\pi}\left(i\right)}\right) \neq h^*\left(x_{\boldsymbol{\pi}\left(i\right)}\right)\right] +\\
2\mathbb{P}_{\boldsymbol{\sigma}\left(\boldsymbol{\pi}\left(i\right)\right) | S^{\boldsymbol{\beta},\boldsymbol{\sigma}}_5,\boldsymbol{\beta}}\left(\boldsymbol{\sigma}\left(\boldsymbol{\pi}\left(i\right)\right) = \boldsymbol{\pi}(i)+n \big| S^{\boldsymbol{\beta},\boldsymbol{\sigma}}_5,\boldsymbol{\beta}\right)\cdot\1 \left[\mathcal{B}\left(S_5^{\boldsymbol{\beta},\boldsymbol{\sigma}}\right)\left(x_{\boldsymbol{\pi}(i)+n}\right) \neq h^*\left(x_{\boldsymbol{\pi}(i)+n}\right)\right] \bigg|\boldsymbol{\beta}\bigg]\Bigg].
\end{multline*}
Note that conditional on $\boldsymbol{\beta}$, the uniform random choice of the test data point $x_{\boldsymbol{\sigma}\left(\boldsymbol{\pi}\left(i\right)\right)} \notin \Gamma_5^{\boldsymbol{\pi}}$ is independent of the uniform random selection of the training data set $S_5^{\boldsymbol{\beta},\boldsymbol{\sigma}}$ induced by the swap variables in $\Gamma_5^{\boldsymbol{\pi}}$, for all $i\in [\frac{4n}{5}]$, we can write
\begin{multline*}
\label{eq:cond_indep_sigma}
2\cdot \E_{\boldsymbol{\beta}}\Bigg[\sum_{i=1}^{\frac{4n}{5}} \E_{\boldsymbol{\sigma}|\boldsymbol{\beta}}\bigg[ \E_{\boldsymbol{\sigma}(\boldsymbol{\pi}(i))|S^{\boldsymbol{\beta},\boldsymbol{\sigma}}_5, \boldsymbol{\beta}}\bigg[\1 \Big[\boldsymbol{\sigma}\left(\boldsymbol{\pi}\left(i\right)\right) = \boldsymbol{\pi}(i) \Big] \cdot \1 \left[ \mathcal{B}\left(S_5^{\boldsymbol{\beta},\boldsymbol{\sigma}}\right)\left(x_{\boldsymbol{\pi}\left(i\right)}\right) \neq h^*\left(x_{\boldsymbol{\pi}\left(i\right)}\right) \right] +\\
       \1\Big[\boldsymbol{\sigma}\left(\boldsymbol{\pi}\left(i\right)\right) = \boldsymbol{\pi}(i)+n\Big] \cdot \1 \left[\mathcal{B}\left(S_5^{\boldsymbol{\beta},\boldsymbol{\sigma}}\right)\left(x_{\boldsymbol{\pi}(i)+n}\right) \neq h^*\left(x_{\boldsymbol{\pi}(i)+n}\right)\right]  \bigg| S^{\boldsymbol{\beta},\boldsymbol{\sigma}}_5, \boldsymbol{\beta}\bigg] \bigg| \boldsymbol{\beta} \bigg] \Bigg].
\end{multline*}
The above expression can be further simplified to \eqref{eq:cond_expectation}, where the inner conditional expectation of the 0-1 loss over $\boldsymbol{\sigma}(\boldsymbol{\pi}(i))$ given $S_5^{\boldsymbol{\beta},\boldsymbol{\sigma}},\boldsymbol{\beta}$ is equal to the expected 0-1 loss of a fixed arbitrary classifier $h_5^{\boldsymbol{\beta},\boldsymbol{\sigma}}$ at $x_{\boldsymbol{\sigma}(\boldsymbol{\pi}(i))}$ over the random swaps of $\boldsymbol{\sigma}(\boldsymbol{\pi}(i))$.
And the outer conditional expectation over $\boldsymbol{\sigma}$ given $\boldsymbol{\beta}$ is to average over all classifiers  $h_5^{\boldsymbol{\beta},\boldsymbol{\sigma}}$ output by the base learner under random swaps, such expected 0-1 loss at $x_{\boldsymbol{\sigma}\left(\boldsymbol{\pi}\left(i\right)\right)}$, for any $i \in [\frac{4n}{5}]$.
\begin{equation}\label{eq:cond_expectation}
        2\cdot \E_{\boldsymbol{\beta}}\left[\sum_{i=1}^{\frac{4n}{5}} \E_{\boldsymbol{\sigma}|\boldsymbol{\beta}}\left[\E_{{\boldsymbol{\sigma}(\boldsymbol{\pi}(i))|S^{\boldsymbol{\beta},\boldsymbol{\sigma}}_5, \boldsymbol{\beta}}}\left[\1 \left[ \mathcal{B}\left(S_5^{\boldsymbol{\beta},\boldsymbol{\sigma}}\right)(x_{\boldsymbol{\sigma}(\boldsymbol{\pi}(i))}) \neq h^*(x_{\boldsymbol{\sigma}(\boldsymbol{\pi}(i))})\right] \bigg|S^{\boldsymbol{\beta},\boldsymbol{\sigma}}_5,\boldsymbol{\beta}\right]\bigg| \boldsymbol{\beta}\right]\right].
\end{equation}
By the law of total conditional expectation and linearity of conditional expectations, 
we have that the above expression equals two times the expected number of errors of the base learner over $\boldsymbol{\beta}, \boldsymbol{\sigma}$ among the other $\frac{4n}{5}$ training data variables, as shown in \eqref{eq:exp_training errors}. 
Hence it suffices to count the expected number of errors on the training side.
\begin{align*}
2\cdot\E_{\boldsymbol{\beta}}\left[\sum_{i=1}^{\frac{4n}{5}} \E_{\boldsymbol{\sigma}|\boldsymbol{\beta}}\left[\1 \left[ \mathcal{B}\left(S_5^{\boldsymbol{\beta},\boldsymbol{\sigma}}\right)\left(x_{\boldsymbol{\sigma}(\boldsymbol{\pi}(i))}\right) \neq h^*\left(x_{\boldsymbol{\sigma}(\boldsymbol{\pi}(i))}\right)\right] \bigg| \boldsymbol{\beta}\right] \right]
\end{align*}
\begin{equation}\label{eq:exp_training errors}
=2\cdot\E_{\boldsymbol{\beta}}\left[\E_{\boldsymbol{\sigma}|\boldsymbol{\beta}}\Bigg[\sum_{i=1}^{\frac{4n}{5}} \1 \left[ \mathcal{B}\left(S_5^{\boldsymbol{\beta},\boldsymbol{\sigma}}\right)(x_{\boldsymbol{\sigma}(\boldsymbol{\pi}(i))}) \neq h^*(x_{\boldsymbol{\sigma}(\boldsymbol{\pi}(i))})\right] \bigg|\boldsymbol{\beta} \Bigg]\right].
\end{equation}
By Fubini's theorem, we can swap the order of expectation and condition on $\boldsymbol{\sigma}$. This gives,
\begin{align*}
        &2\cdot\E_{\boldsymbol{\sigma}}\left[\E_{\boldsymbol{\beta}| \boldsymbol{\sigma}}\Bigg[\sum_{i=1}^{\frac{4n}{5}} \1 \left[ \mathcal{B}\left(S_5^{\boldsymbol{\beta},\boldsymbol{\sigma}}\right)(x_{\boldsymbol{\sigma}(\boldsymbol{\pi}(i))}) \neq h^*(x_{\boldsymbol{\sigma}(\boldsymbol{\pi}(i))}) \right] \bigg|\boldsymbol{\sigma} \Bigg] \right].
\end{align*}
Since the expected number of errors over $\boldsymbol{\beta}$ given $\boldsymbol{\sigma}$ among $\frac{4n}{5}$ training data points is at most that among $n$ training data points, it follows that the above expression is bounded by
\begin{align*}
        &2\cdot\E_{\boldsymbol{\sigma}}\left[\E_{\boldsymbol{\beta}| \boldsymbol{\sigma}}\bigg[\sum_{i=1}^{n} \1 \left[ \mathcal{B}\left(S_5^{\boldsymbol{\beta},\boldsymbol{\sigma}}\right)(x_{\boldsymbol{\sigma}(\boldsymbol{\pi}(i))}) \neq h^*(x_{\boldsymbol{\sigma}(\boldsymbol{\pi}(i))})\right] \bigg|\boldsymbol{\sigma}\bigg]\right].
\end{align*}
Recall that the randomness of $\boldsymbol{\beta}$ comes from the random permutation $\boldsymbol{\pi}$ and duplicates adjustment with an i.i.d. sequence $\boldsymbol{U}$ drawn from $\Unif{\{1,\dots,n\}}$.
According to Lemma $\ref{lemma:duplicates_adjust}$, after adjusting for duplicates in $S_5^{\boldsymbol{\pi},\boldsymbol{\sigma}}$,
we have a conditionally i.i.d. data set $S_5^{\boldsymbol{\beta},\boldsymbol{\sigma}}$ from $\Unif{\sett{x_{\boldsymbol{\sigma}(1)},\dots,x_{\boldsymbol{\sigma}(n)}}}$ given $\boldsymbol{\sigma}$.
%
Therefore, the above expression is bounded by two times the i.i.d. guarantee on the expected number of errors of the base learner among $n$ training data points, as follows:
\begin{align*}
        &2n \cdot \E_{\boldsymbol{\sigma}}\left[\E_{\boldsymbol{\beta} | \boldsymbol{\sigma}}\bigg[\frac{1}{n} \sum_{i=1}^{n} \1 \left[ \mathcal{B}\left(S_5^{\boldsymbol{\beta},\boldsymbol{\sigma}}\right)(x_{\boldsymbol{\sigma}(\boldsymbol{\pi}(i))}) \neq h^*(x_{\boldsymbol{\sigma}(\boldsymbol{\pi}(i))})\right]\bigg|\boldsymbol{\sigma}\bigg]\right]\\
        &=2n \cdot \E_{\boldsymbol{\sigma}}\left[\E_{\boldsymbol{\beta}| \boldsymbol{\sigma}}\left[\text{er}_{\text{Unif}\left(x_{\boldsymbol{\sigma}(1)},\dots,x_{\boldsymbol{\sigma}(n)}\right) | \boldsymbol{\sigma}} \left(\mathcal{B}\left(S_5^{\boldsymbol{\beta},\boldsymbol{\sigma}}\right) \right)\bigg|\boldsymbol{\sigma}\right]\right]\\
        &\leq 2n \cdot \E_{\boldsymbol{\sigma}}\left[R\left(\frac{n}{5}\right)\right]\\
        &= 2nR\left(\frac{n}{5}\right).
\end{align*}
as claimed.
\end{proof}

\subsection{Expected size of the conditional error region}
In the following lemma, we apply a Markov's inequality to derive a guarantee on the expected size of the conditional error regions for each base learner, using the guarantee on the expected number of errors among the other $\frac{4n}{5}$ pairs of data points, estabilished in Lemma \ref{lemma:expected_error}.
\begin{lemma}\label{expected_err_region}
   Define a partition of $\sz_x \setminus \Gamma_k^{\boldsymbol{\pi}}$ based on the conditional probability of error of the $k$-th base learner over random swaps given permutations and duplicates adjustments. Consider the following set for any $k\in [5]$ and any $j\in \mathbb{N}$:
   \begin{align*}
     R_{j}^{k,\boldsymbol{\beta}} &:= \left\{ x\notin \Gamma_k^{\boldsymbol{\pi}}: \mathbb{P}_{\boldsymbol{\sigma}|\boldsymbol{\beta}}\left(\mathcal{B}\left(S_k^{\boldsymbol{\beta},\boldsymbol{\sigma}}\right)\left(x\right) \neq h^*\left(x\right)\big|\boldsymbol{\beta} \right)\in \left(2^{-j},2^{-j+1}\right]\right\}.
    \end{align*}
    Then we have that the expected size of the $j$-th conditional error region of the $k$-th base learner over random permutation and duplicates adjustment is at most $2^j\cdot 2nR\left(\frac{n}{5}\right)$. That is, for every $j\in \mathbb{N}$,
    \[
    \mathbb{E}_{\boldsymbol{\beta}}\big[\big|R_{j}^{k,\boldsymbol{\beta}}\big|\big] \leq 2^j\cdot 2nR\left(\frac{n}{5}\right).
    \]
\end{lemma}
\begin{proof}
Fix any realizable $\sz \in \mathcal{Z}^{2n}$.
The number of data points with conditional probability of error in the range $\left(2^{-j},2^{-j+1}\right]$ is at most that with conditional probability of error at least $2^{-j}$ for any $j\in \mathbb{N}$. Thus, we have the expected size of the $j$-th conditional error region of $k$-th base learner bounded by
\begin{align*}
    \mathbb{E}_{\boldsymbol{\beta}}\big[\big|R_{j}^{k,\boldsymbol{\beta}}\big|\big]&= \mathbb{E}_{\boldsymbol{\beta}}\bigg[\sum_{x \notin \Gamma_k^{\boldsymbol{\pi}}} \1\left[\mathbb{P}_{\boldsymbol{\sigma}| \boldsymbol{\beta}}\left(\mathcal{B}\left(S_k^{\boldsymbol{\beta},\boldsymbol{\sigma}}\right)\left(x\right) \neq h^*(x)\big| \boldsymbol{\beta}\right) \in \left(2^{-j},2^{-j+1}\right]\right]\bigg]\\
    &\leq \mathbb{E}_{\boldsymbol{\beta}}\bigg[\sum_{x \notin \Gamma_k^{\boldsymbol{\pi}}} \mathbbm{1}\left[\mathbb{P}_{\boldsymbol{\sigma}| \boldsymbol{\beta}}\left(\mathcal{B}\left(S_k^{\boldsymbol{\beta},\boldsymbol{\sigma}}\right)\left(x\right) \neq h^*(x)| \boldsymbol{\beta}\right) > 2^{-j}\right]\bigg].
\end{align*}
Note that both the training data set $S_k^{\boldsymbol{\beta},\boldsymbol{\sigma}}$ and test data points $x\in \sz_x\textbackslash\Gamma_k^{\boldsymbol{\pi}}$ are dependent on random permutation $\boldsymbol{\pi}$. With an application of Markov's inequality to the indicator function, the above expression can be bounded by
\begin{align*}
     &\mathbb{E}_{\boldsymbol{\beta}}\bigg[\sum_{x \notin \Gamma_k^{\boldsymbol{\pi}}} 2^{j}\cdot \mathbb{P}_{\boldsymbol{\sigma}| \boldsymbol{\beta}}\left(\mathcal{B}\left(S_k^{\boldsymbol{\beta},\boldsymbol{\sigma}}\right)\left(x\right) \neq h^*(x)\big| \boldsymbol{\beta}\right)\bigg].
\end{align*}
Note that $\1\left[x_i\notin \Gamma_k^{\boldsymbol{\pi}}\right]$ is ${\sigma}(\boldsymbol{\beta})$-measurable. By the linearity of conditional expectation, we have the above expression equals
\begin{align*}
   &\mathbb{E}_{\boldsymbol{\beta}}\bigg[\sum_{i=1}^{2n} 2^{j}\cdot\1\big[x_i\notin \Gamma_k^{\boldsymbol{\pi}}\big]\cdot \mathbb{P}_{\boldsymbol{\sigma}|\boldsymbol{\beta}}\left(\mathcal{B}\left(S_k^{\boldsymbol{\beta},\boldsymbol{\sigma}}\right)\left(x_i\right) \neq h^*(x_i) \Big|\boldsymbol{\beta}\right)\bigg]\\
   &= 2^{j}\cdot\mathbb{E}_{\boldsymbol{\beta}}\Bigg[ \E_{\boldsymbol{\sigma}|  \boldsymbol{\beta}}\bigg[\sum_{i=1}^{2n}\1\Big[x_i\notin \Gamma_k^{\boldsymbol{\pi}}\Big]\1\left[\mathcal{B}\left(S_k^{\boldsymbol{\beta},\boldsymbol{\sigma}}\right)\left(x_i\right) \neq h^*(x_i)\right] \Big| \boldsymbol{\beta}\bigg]\Bigg].
\end{align*}
By the law of total expectation, we have the above expression equivalent to $2^j$ times the guarantee on the expected number of errors of the base learner over $\boldsymbol{\beta},\boldsymbol{\sigma}$ among the other $\frac{4n}{5}$ pairs of data points, as follows:
\begin{align*}
&2^{j}\cdot\mathbb{E}_{\boldsymbol{\beta},\boldsymbol{\sigma}}\Bigg[\sum_{i=1}^{2n}\1\Big[x_i\notin \Gamma_k^{\boldsymbol{\pi}}\Big]\1\left[\mathcal{B}\left(S_k^{\boldsymbol{\beta},\boldsymbol{\sigma}}\right)\left(x_i\right) \neq h^*(x_i)\right] \Bigg]\\
    &= 2^j\cdot \E_{\boldsymbol{\beta}, \boldsymbol{\sigma}}\bigg[\sum_{x\notin \Gamma_k^{\boldsymbol{\pi}}} \1 \left[\mathcal{B}(S^{\boldsymbol{\beta},\boldsymbol{\sigma}}_k)(x) \neq h^*(x)\right]\bigg].
\end{align*}
According to Lemma~\ref{lemma:expected_error}, we have the above expression bounded by
$2^j\cdot 2nR(\frac{n}{5})$ as claimed.
\end{proof}

\section{Conditional entropy of a randomized Majority-of-5 learner}
With all the ingredients above, in the following lemma, we derive a guarantee on the conditional entropy of errors of the randomized Majority-of-5 learner on a fixed arbitrary realizable supersample $\sz\in \mathcal{Z}^{2n}$, providing an upper bound on the worst-case eCMI of the learner.
\begin{lemma}\label{lemma:cond enty of maj-of-5}
The conditional entropy of errors of the output classification of a randomized Majority-of-5 learner $A$ given permutations and duplicates adjustments, on $2n$ data points in a realizable supersample $\sz \in \mathcal{Z}^{2n}$ is
\begin{align*}
    \sum_{i\leq 2n} H\Big(\1\big[A(\sz_{\boldsymbol{\sigma}},\boldsymbol{\beta})\left(x_i\right) \neq h^*(x_i)\big] \big| \boldsymbol{\beta}\Big) = O\left(nR\left(\frac{n}{5}\right)\right).
\end{align*}
\end{lemma}
\begin{proof}
Fix any realizable $\sz \in \mathcal{Z}^{2n}$.
Since the binary entropy function $h(p)= -p\log p - (1-p)\log (1-p)$ is bounded by the function $f(p)= -p\log p + 2p = p\log\frac{4}{p}$ on $[0,1]$, we have the conditional entropy of errors given $\boldsymbol{\beta}$ of $A$ on the supersample bounded by
\begin{equation}
    \sum_{i\leq 2n}\E_{\boldsymbol{\beta}}\Bigg[\mathbb{P}_{\boldsymbol{\sigma} | \boldsymbol{\beta}}\left(A(\sz_{\boldsymbol{\sigma}},\boldsymbol{\beta})(x_i)\neq h^*(x_i) | \boldsymbol{\beta}\right)  \log \frac{4}{ \mathbb{P}_{\boldsymbol{\sigma} | \boldsymbol{\beta}}\left(A(\sz_{\boldsymbol{\sigma}},\boldsymbol{\beta})(x_i)\neq h^*(x_i) | \boldsymbol{\beta}\right)  } \Bigg].
\label{eq:cond_entropy}
\end{equation}
Without loss of generality, we consider the first part of the partition $\Gamma_1^{\boldsymbol{\pi}}$ of the supersample $\sz$.
For any $x_i \in \Gamma_1^{\boldsymbol{\pi}}$ in the first part of the partition of $\sz$,
we define the event $L_i(h_k^{\boldsymbol{\beta},\boldsymbol{\sigma}})$ over $S_k^{\boldsymbol{\beta},\boldsymbol{\sigma}}$ that the $k$-th base learner errs at $x_i$:
\begin{align*}
    L_i\left(h_k^{\boldsymbol{\beta},\boldsymbol{\sigma}}\right) := \left\{ h_k^{\boldsymbol{\beta},\boldsymbol{\sigma}}(x_i)\neq h^*(x_i)\right\}.
\end{align*}
From the analysis in Lemma \ref{lemma:expected_error}, we know that the $k$-th base learner running in the $k$-th part of the partition might not have a guarantee on the conditional probability of error over random swaps at data points in the $k$-th partition. 
Nevertheless, at least two of the other four base classifiers are not trained on the data points in the $k$-th part of the partition.
In fact, we have that at least two of the other four base classifiers must err when a majority-of-5 classifier errs.

Therefore, 
we can bound the conditional probability of error over $\boldsymbol{\sigma}$ given $\boldsymbol{\beta}$ of the Majority-of-5 learner at $x_i$, by the probability that at least two of the other four base learners both err at $x_i$. That is, for any $x_i \in \Gamma_1^{\boldsymbol{\pi}}$,
\begin{multline*}
    \mathbb{P}_{\boldsymbol{\sigma} | \boldsymbol{\beta}}\left(A(\sz_{\boldsymbol{\sigma}},\boldsymbol{\beta})(x_i)\neq h^*(x_i) | \boldsymbol{\beta}\right) \leq
    \mathbb{P}_{\boldsymbol{\sigma} | \boldsymbol{\beta}}\left(\exists\, 2\leq k'<k\leq 5: h_k^{\boldsymbol{\beta},\boldsymbol{\sigma}}(x_i)\neq h^*(x_i) \wedge h_{k'}^{\boldsymbol{\beta},\boldsymbol{\sigma}}(x_i)\neq h^*(x_i) |\boldsymbol{\beta} \right).
\end{multline*}
By a union bound, we have
\begin{align*}
    \mathbb{P}_{\boldsymbol{\sigma} | \boldsymbol{\beta}}\left(A(\sz_{\boldsymbol{\sigma}},\boldsymbol{\beta})(x_i)\neq h^*(x_i) | \boldsymbol{\beta}\right)
    \leq 
    \sum_{\substack{2 \leq k'<k \leq 5}}\mathbb{P}_{\boldsymbol{\sigma}| \boldsymbol{\beta}}\left(L_i(h_k^{\boldsymbol{\beta},\boldsymbol{\sigma}}) \wedge 
    L_i(h_{k'}^{\boldsymbol{\beta},\boldsymbol{\sigma}})|\boldsymbol{\beta}\right).
\end{align*}
By the symmetry of a uniform random permutation over $\sz_{\boldsymbol{\sigma}}$,
the conditional entropy of errors on $2n$ data points is five times that on the data points in $\Gamma^{\boldsymbol{\pi}}_1$. And $f(p) = p\log \frac{4}{p}$ is monontone non-decreasing on $[0,1]$, we have \eqref{eq:cond_entropy} bounded by
\begin{equation}
     5 \sum_{i\leq 2n}\E_{\boldsymbol{\beta}}\Bigg[\1[x_i\in \Gamma^{\boldsymbol{\pi}}_1 ](\sum_{\substack{2 \leq k'<k \leq 5}}\mathbb{P}_{\boldsymbol{\sigma}| \boldsymbol{\beta}}(L_i(h_k^{\boldsymbol{\beta},\boldsymbol{\sigma}}) \wedge 
    L_i(h_{k'}^{\boldsymbol{\beta},\boldsymbol{\sigma}})|\boldsymbol{\beta})) \log \frac{4}{ \min \{1, \sum_{\substack{2 \leq k'<k \leq 5}}\mathbb{P}_{\boldsymbol{\sigma}| \boldsymbol{\beta}}(L_i\left(h_k^{\boldsymbol{\beta},\boldsymbol{\sigma}}\right) \wedge L_i(h_{k'}^{\boldsymbol{\beta},\boldsymbol{\sigma}})|\boldsymbol{\beta}) \} }\Bigg].
\label{eq:error_first_block}
\end{equation}
Define the $j$-th conditional error region on the data points in the first part of the partition $\Gamma^{\boldsymbol{\pi}}_1$ based on the conditional probability of error of the $k$-th base learner over $\boldsymbol{\sigma}$ given $\boldsymbol{\beta}$:
\begin{align*}
    R^{k,\boldsymbol{\beta},1}_j := \left\{x \in \Gamma^{\boldsymbol{\pi}}_1: \mathbb{P}_{\boldsymbol{\sigma}|\boldsymbol{\beta}}(h_k^{\boldsymbol{\beta},\boldsymbol{\sigma}}(x) \neq h^*(x) | \boldsymbol{\beta}) \in (2^{-j}, 2^{-j+1}] \right\}.
\end{align*}
Let $j^{\boldsymbol{\beta}}_k(x_i)$ denote the index of the conditional error region such that the conditional probability of error of the $k$-th base learner at $x_i \in \Gamma^{\boldsymbol{\pi}}_1 $, is in the interval $(2^{-j^{\boldsymbol{\beta}}_k(x_i)}, 2^{-j^{\boldsymbol{\beta}}_k(x_i)+1}]$. 
By the linearity of expectation and the conditional independence of $S_k^{\boldsymbol{\beta},\boldsymbol{\sigma}}, S_{k'}^{\boldsymbol{\beta},\boldsymbol{\sigma}}$ given $\boldsymbol{\beta}$, we have \eqref{eq:error_first_block} bounded by
\begin{align*}
    &5 \sum_{i\leq 2n}\E_{\boldsymbol{\beta}}\Big[\1[x_i\in \Gamma^{\boldsymbol{\pi}}_1]
    (\sum_{\substack{2 \leq k'<k \leq 5}} 2^{-j^{\boldsymbol{\beta}}_k(x_i) -j^{\boldsymbol{\beta}}_{k'}(x_i)+2}) 
    \log \frac{4}{ \min\{1, \sum_{\substack{2 \leq k'<k \leq 5}} 2^{-j^{\boldsymbol{\beta}}_k(x_i)-j^{\boldsymbol{\beta}}_{k'}(x_i) +2 }\}}\Big]\\
    &\leq 5 \sum_{\substack{2 \leq k'<k \leq 5}} \E_{\boldsymbol{\beta}}\Big[\sum_{x_i \in \Gamma^{\boldsymbol{\pi}}_1 }
    (2^{-j^{\boldsymbol{\beta}}_k(x_i) -j^{\boldsymbol{\beta}}_{k'}(x_i)+2} ) 
    \log \frac{4}{2^{-j^{\boldsymbol{\beta}}_k(x_i)-j^{\boldsymbol{\beta}}_{k'}(x_i)+2}}\Big].
\end{align*}
Combined with the inequality $xy\leq \max\{x^2 , y^2\}$, the above expression is bounded by
\begin{align*}
    &5 \sum_{\substack{2 \leq k'<k \leq 5}} \E_{\boldsymbol{\beta}}\Big[\sum_{x_i \in \Gamma^{\boldsymbol{\pi}}_1 }
    \max\{ 2^{-2j^{\boldsymbol{\beta}}_k(x_i)+2} , 2^{-2j^{\boldsymbol{\beta}}_{k'}(x_i)+2} \}
    \log \frac{4}{ \max\{ 2^{-2j^{\boldsymbol{\beta}}_k(x_i)+2} , 2^{-2j^{\boldsymbol{\beta}}_{k'}(x_i)+2} \} }\Big].
\end{align*}
Since $f(p)= p \log \frac{4}{p}$ is non-negative on $[0,1]$, it follows that $f(\max\{x^2, y^2\})\leq f(x^2)+f(y^2)$.
We have the above expression bounded by
\begin{align*}
    &5  \sum_{\substack{2 \leq k'<k \leq 5}} \E_{\boldsymbol{\beta}}\Big[\sum_{x_i \in \Gamma^{\boldsymbol{\pi}}_1 }
    2^{-2j^{\boldsymbol{\beta}}_k(x_i)+2} \log \frac{4}{2^{-2j^{\boldsymbol{\beta}}_k(x_i)+2}} +
     2^{-2j^{\boldsymbol{\beta}}_{k'}(x_i)+2} \log \frac{4}{2^{-2j^{\boldsymbol{\beta}}_{k'}(x_i)+2}}\Big]\\
    &\leq 5 \sum_{\substack{2 \leq k'<k \leq 5}}\Big( \E_{\boldsymbol{\beta}}[\sum_{j=1}^{\infty}|R^{k,\boldsymbol{\beta},1}_j|\cdot
    2^{-2j+2}\cdot 2j] +
     \E_{\boldsymbol{\beta}}[\sum_{j=1}^{\infty}|R^{k',\boldsymbol{\beta},1}_j|\cdot
    2^{-2j+2}\cdot 2j] \Big)\\
    &= 120  \sum_{k=2}^5\sum_{j=1}^{\infty}\E_{\boldsymbol{\beta}}[|R^{k,\boldsymbol{\beta},1}_j|]\cdot
    2^{-2j}\cdot j 
    \end{align*}
    According to Lemma \ref{expected_err_region}, since the $j$-th conditional error region in the first part of the partition $\Gamma^{\boldsymbol{\pi}}_1$ is a subset of the $j$-th conditional error region in the other $\frac{4n}{5}$ pairs of data points, i.e.,
    $R_j^{k,\boldsymbol{\beta},1} \subseteq R_j^{k,\boldsymbol{\beta}}$, we have
    $\E_{\boldsymbol{\beta}}[|R^{k,\boldsymbol{\beta},1}_j|] \leq \E_{\boldsymbol{\beta}}[|R^{k,\boldsymbol{\beta}}_j|] \leq 2^j\cdot 2nR(\frac{n}{5})$, for any $ 2\leq k\leq 5$.
    Thus, the above expression is bounded by
    \begin{align*}
     480 \cdot \Big(\sum_{j=1}^{\infty} j2^{-j}\Big) \cdot 2nR\left(\frac{n}{5}\right)
    = 1920 \cdot nR\left(\frac{n}{5}\right)
    =O\Big(nR\left(\frac{n}{5}\right)\Big).
\end{align*}
This concludes the proof.
\end{proof} 

\vspace{-1.8em}
\section{Expected training error rate and Proof of the Main Theorem}
Now we are ready to derive the guarantee on the worst-case eCMI for the randomized Majority-of-5 learner and its expected training error rate, which recovers an optimal PAC guarantee on the expected generalization error of the algorithm for learning VC classes in the realizable case.

Since we are only concerned with the training error rate, the random swap variables are not relevant. Hence, we may assume $\boldsymbol{\sigma}(i)=i$ for all $i\in[n]$ for the analysis in this section.
\begin{proof}[Proof of Theorem~\ref{thm:eCMI(A)}] 
We will show that the worst-case eCMI of the constructed randomized Majority-of-5 learner is of the same order as the i.i.d. guarantee on the expected number of errors of the base learner among $n$ training data points.
Combining Lemma~\ref{lemma:epy_of_points} and Lemma~\ref{lemma:cond enty of maj-of-5}, we have
\begin{align*}
     \text{eCMI}(\ell(A)) &= \sup_{{\sz\in \mathcal{Z}^{2n} \text{realizable}}} I\left(\vec{\ell}\left(A\left(\sz_{\boldsymbol{\sigma}}, \boldsymbol{\beta}\right), \sz\right) ; \boldsymbol{\sigma}\right) \\
     &\leq \sup_{{\sz\in \mathcal{Z}^{2n}\text{realizable} }} \sum_{i\leq 2n} H\left(\ell\left(A(\sz_{\boldsymbol{\sigma}},\boldsymbol{\beta}),\sz_{i}\right) | \boldsymbol{\beta}\right)\\
     &= \sup_{{\sz\in \mathcal{Z}^{2n}\text{realizable}}}\sum_{i\leq 2n} H\left(\1\left[A(\sz_{\boldsymbol{\sigma}},\boldsymbol{\beta})(x_i) \neq h^*(x_i)\right] \big| \boldsymbol{\beta}\right) \\
     &= O\left(nR\left(\frac{n}{5}\right)\right).
\end{align*}
Fix an arbitrary realizable supersample $\sz \in \mathcal{Z}^{2n}$.
For any $x\in \sz_{x}$ and fixed training data sets $S^{\boldsymbol{\beta}, \boldsymbol{\sigma}}_1,S^{\boldsymbol{\beta}, \boldsymbol{\sigma}}_2,\dots,S^{\boldsymbol{\beta}, \boldsymbol{\sigma}}_5$, if $\text{Maj}(h^{\boldsymbol{\beta}, \boldsymbol{\sigma}}_1,h^{\boldsymbol{\beta}, \boldsymbol{\sigma}}_2,h^{\boldsymbol{\beta}, \boldsymbol{\sigma}}_3,h^{\boldsymbol{\beta}, \boldsymbol{\sigma}}_4,h^{\boldsymbol{\beta}, \boldsymbol{\sigma}}_5)(x)\neq h^*(x)$, then at least one base classifier indexed by $k\in[5]$ such that $h^{\boldsymbol{\beta}, \boldsymbol{\sigma}}_k(x)\neq h^*(x)$. 
Then the training error rate of a majority-of-5 classifier is
\begin{align*}
        \ell(A(\sz_{\boldsymbol{\sigma}},\boldsymbol{\beta}), \sz_{\boldsymbol{\sigma}} )
        =&\mathbb{P}_{X\sim \Unif{\sett{x_{{\sigma}(1)},\dots,x_{{\sigma}(n)}}}}\Big(\text{Maj}(h^{\boldsymbol{\beta}, \boldsymbol{\sigma}}_1,h^{\boldsymbol{\beta}, \boldsymbol{\sigma}}_2,h^{\boldsymbol{\beta}, \boldsymbol{\sigma}}_3,h^{\boldsymbol{\beta}, \boldsymbol{\sigma}}_4,h^{\boldsymbol{\beta}, \boldsymbol{\sigma}}_5)(X)\neq h^*(X)\Big)\\
        \leq&\mathbb{P}_{X\sim \Unif{\sett{x_{{\sigma}(1)},\dots,x_{{\sigma}(n)}}}}
        \left(\exists k\in[5]:h^{\boldsymbol{\beta}, \boldsymbol{\sigma}}_k(X)\neq h^*(X)\right)\\
        \leq&\sum_{k\in[5]}\mathbb{P}_{X\sim \Unif{\sett{x_{{\sigma}(1)},\dots,x_{{\sigma}(n)}}}}\left(h^{\boldsymbol{\beta}, \boldsymbol{\sigma}}_k(X)\neq h^*(X)\right).
\end{align*}
Taking the expectation with respect to the randomness of the algorithm $\boldsymbol{\beta}$, we have
\begin{align*}
    \E_{\boldsymbol{\beta}}\Big[\ell(A(\sz_{\boldsymbol{\sigma}},\boldsymbol{\beta}), \sz_{\boldsymbol{\sigma}})\Big] &\leq \E_{\boldsymbol{\beta}}\Big[\sum_{k\in[5]}\mathbb{P}_{X\sim \Unif{\sett{x_{\boldsymbol{\sigma}(1)},\dots,x_{\boldsymbol{\sigma}(n)}}}}(h^{\boldsymbol{\beta}, \boldsymbol{\sigma}}_k(X)\neq h^*(X))\Big]\\
    &\leq \sum_{k\in[5]} \E_{\boldsymbol{\beta}}\Big[\mathbb{P}_{X\sim \Unif{\sett{x_{\boldsymbol{\sigma}(1)},\dots,x_{\boldsymbol{\sigma}(n)}}}}\big(h^{\boldsymbol{\beta}, \boldsymbol{\sigma}}_k(X)\neq h^*(X)\big)\Big]\\
    &\leq 5 \cdot R\left(\frac{n}{5}\right).
\end{align*}
This concludes the proof of the theorem.
\end{proof}
\begin{proof}[Proof of Theorem~\ref{thm:eCMI(maj-of-3)}]
Fix an arbitrary realizable supersample $\sz \in \mathcal{Z}^{2n}$.
We assume the base learner achieves an optimal in-expectation generalization guarantee,  
as shown for the Majority-of-3 learner in \cite{pmlr-v247-aden-ali24a}.
According to Theorem~\ref{thm:eCMI(A)}, for a universal constant $c\geq 1$, we have 
the expected training error rate of a randomized Majority-of-5 learner with respect to the randomness of the algorithm as
\begin{align*}
     \E_{\boldsymbol{\beta}}\left[\ell(A(\sz_{\boldsymbol{\sigma}},\boldsymbol{\beta}), \sz_{\boldsymbol{\sigma}}) \right] \leq 5 \cdot c\frac{d}{n} = O\left(\frac{d}{n}\right).
\end{align*}
And we have that the worst-case eCMI of a randomized Majority-of-5 learner with the Majority-of-3 learner as the base learner as
\begin{align*}
    \text{eCMI}\left(\ell(\text{Maj}(h^{\boldsymbol{\beta},\boldsymbol{\sigma}}_1,h^{\boldsymbol{\beta},\boldsymbol{\sigma}}_2,h^{\boldsymbol{\beta}, \boldsymbol{\sigma}}_3,h^{\boldsymbol{\beta},\boldsymbol{\sigma}}_4,h^{\boldsymbol{\beta},\boldsymbol{\sigma}}_5))\right) \leq 1920 \cdot  n\cdot c\frac{d}{n} = O\left(d\right).
\end{align*}
This completes the proof of the main theorems.
\end{proof}

\section{Distributional eCMI}\label{distributional eCMI}
In previous sections, we present a randomized Majority-of-5 learning algorithm that achieves an optimal guarantee on the worst-case eCMI. 
In this section, we show that the extra randomness of the algorithm (i.e. random permutation and duplicates adjustment) is unnecessary for obtaining a generalization guarantee with a weaker information complexity measure.
Moreover, we show that a deterministic Majority-of-3 base learners with optimal in-expectation generalization guarantee is sufficient for an optimal guarantee on the distributional eCMI. 
In particular, we show the following theorem.

\begin{theorem}\label{thm:distributional eCMI(A)}
For any base learner $\mathcal{B}: (\mathcal{X} \times \{0,1\})^{m} \to \{0,1\}^{\mathcal{X}}$ having guarantee that $\forall \mathcal{D}$ realizable, for i.i.d. data set $S_m \sim \mathcal{D}^m$, the expected true error rate and expected training error rate satisfies $\max\{\E_{S_m\sim \mathcal{D}^m}[\text{er}_{S_m}(\mathcal{B}(S_m))],\E_{S_m\sim \mathcal{D}^m}[\text{er}_{\mathcal{D}}(\mathcal{B}(S_m))]\} \leq R(m)$. The distributional eCMI of a Majority-of-3 of such base learners is $\text{eCMI}_{\mathcal{D}}(\ell(A)) \leq 192\cdot nR(\frac{n}{3})$.  The expected training error rate of $A$ over a random draw of $n$ training samples
$\E_{\sZ_{\boldsymbol{\sigma}}}\Big[\ell(A(\sZ_{\boldsymbol{\sigma}}), \sZ_{\boldsymbol{\sigma}})\Big] \leq 2 \cdot R\left(\frac{n}{3}\right)$.
\end{theorem}

Similar to Algorithm \ref{algo:m5m3}, the algorithm partitions the $n$ training data points into 3 disjoint pieces $S_1^{\sZ,\boldsymbol{\sigma}},S_2^{\sZ,\boldsymbol{\sigma}},S_3^{\sZ,\boldsymbol{\sigma}}$ consisting of $\frac{n}{3}$ training data points each. Then it runs a base learner on each disjoint training data set to obtain output classifiers $h_1^{\sZ,\boldsymbol{\sigma}},h_2^{\sZ,\boldsymbol{\sigma}},h_3^{\sZ,\boldsymbol{\sigma}}$. Finally, it combines the classifiers with a Majority-of-3.
Note that we rely on a stronger assumption on the base learner, as discussed above.
For instance, Majority-of-3 ERMs satisfy such assumption, in particular, its expected training error rate is $R(m) = O(\frac{d}{n})$, which follows from Lemma 2.1 of \cite{pmlr-v247-aden-ali24a} with independent training samples.
Besides, many learning algorithms have zero training error, such as one-inclusion-graph algorithm $\cite{HAUSSLER1994248}$. 

Interestingly, it only takes 3 base learners to achieve an optimal guarantee on $\text{eCMI}_D(A)$.
The underlying intuition is that in this case, given that the ghost samples are i.i.d. from the distribution $\mathcal{D}$, each base learner, under an additional assumption on its expected training error rate, can have a guarantee on its conditional probability of error over random swaps at all $2n$ data samples such that the conditional error regions of each base learner is defined on the $2n$ data samples. In contrast to Algorithm \ref{algo:m5m3}, the conditional error regions of the $k$-th base learner can only be defined on the other $\frac{4n}{5}$ pairs of data points. 
By the conditional independence of base classifiers, it suffices to have at least two base classifiers err when majority vote classifier errs to compensate for the exponential growth rate of the sequence of expected size of conditional error regions.
Hence,
3 optimal base learners suffices to provide such an optimal guarantee on $\text{eCMI}_D(A)$.
 

The rest of the section is devoted to proving the above theorem.
The analysis is similar to the analysis for Algorithm \ref{algo:m5m3}, involving a reduction to conditional entropy, deriving a guarantee on the expected number of errors on a supersample for each base learner, which implies a guarantee on the expected size of the conditional error regions. However, the overall proof is simpler.
\vspace{-6pt}
\subsection{Expected number of errors on a supersample}
In this algorithm, the training data set $S_k^{\tilde{Z},\boldsymbol{\sigma}}$ consists of $\frac{n}{3}$ i.i.d. samples from the distribution $\mathcal{D}$. In the following lemma, we derive a guarantee on the expected number of errors on the $2n$ data samples for each base learner.
\begin{lemma}\label{lemma:expected_error_iid}
       The expected number of errors of the $k$-th base learner $\mathcal{B}$ among $2n$ data samples over random swaps, a random draw of the supersample for any $k\in[3]$ is
     \[
     \E_{\tilde{Z}, \boldsymbol{\sigma}}\Bigg[\sum_{i = 1}^{2n} \1 \left[\mathcal{B}\left(S_k^{\tilde{Z},\boldsymbol{\sigma}}\right)\left(X_i\right) \neq h^*\left(X_i\right)\right]\Bigg] \leq 2nR\left(\frac{n}{3}\right).
     \]
\end{lemma}
\begin{proof}
Without loss of generality, consider the first base learner with training data set $S_1^{\sZ, \boldsymbol{\sigma}}$.
Since the ghost data samples are i.i.d. drawn from the distribution $\mathcal{D}$, we can then condition on $\boldsymbol{\sigma}$ to obtain an expected number of errors of the base learner on $2n$ i.i.d. samples.
By the law of total expectation, we have
\begin{align*}
&\E_{\boldsymbol{\sigma}}\left[\E_{\tilde{Z}|\boldsymbol{\sigma}}\bigg[\sum_{i=1}^{2n} \1 \left[\mathcal{B}\left(S_1^{\tilde{Z},\boldsymbol{\sigma}}\right)\left(X_i\right) \neq h^*\left(X_i\right)\right] \Big|\boldsymbol{\sigma}\bigg] \right].
\end{align*}
Now, we count the expected number of training errors on $\frac{n}{3}$ training data samples and the expected number of errors on the remaining of the $\frac{5n}{3}$ data samples, with expectation taken over a random draw of the supersample $\sZ \sim \mathcal{D}^{2n}$. This gives
\begin{align*}
        &\E_{\boldsymbol{\sigma}}\left[\E_{\tilde{Z}|\boldsymbol{\sigma}}\Bigg[\sum_{i\in \{\boldsymbol{\sigma}(1),\dots,\boldsymbol{\sigma}(\frac{n}{3})\}} \1 \left[\mathcal{B}\left(S_1^{\tilde{Z},\boldsymbol{\sigma}}\right)\left(X_i\right) \neq h^*\left(X_i\right)\right]  + \sum_{i\notin \{\boldsymbol{\sigma}(1),\dots,\boldsymbol{\sigma}(\frac{n}{3})\}}\1 \left[\mathcal{B}\left(S_1^{\tilde{Z},\boldsymbol{\sigma}}\right)\left(X_i\right) \neq h^*\left(X_i\right)\right] \bigg|\boldsymbol{\sigma} \Bigg]\right].
\end{align*}
By the linearity of conditional expectation, we have the above expression equals
\begin{align*}
    & \E_{\boldsymbol{\sigma}}\left[\E_{\tilde{Z}|\boldsymbol{\sigma}}\Bigg[
    \sum_{\substack{
  i\in \{\boldsymbol{\sigma}(1),\dots,\\
  \boldsymbol{\sigma}(\frac{n}{3})\}
    }} 
    \1 \left[\mathcal{B}\left(S_1^{\tilde{Z},\boldsymbol{\sigma}}\right)\left(X_{i}\right) \neq h^*\left(X_{i}\right)\right] \bigg|\boldsymbol{\sigma}\Bigg] + 
    \sum_{\substack{
  i\notin \{\boldsymbol{\sigma}(1),\dots,\\
  \boldsymbol{\sigma}(\frac{n}{3})\}
    }}\E_{\tilde{Z}|\boldsymbol{\sigma}}\left[\1 \left[\mathcal{B}\left(S_1^{\tilde{Z},\boldsymbol{\sigma}}\right)\left(X_{i}\right) \neq h^*\left(X_{i}\right)\right] \bigg|\boldsymbol{\sigma}\right] \right].
\end{align*}
Since the training data set $S_1^{\tilde{Z},\boldsymbol{\sigma}}$ is sampled independently of the test sample $X_i \sim \mathcal{D}$, 
and the expected training error rate of the base learner is at most $R(\frac{n}{3})$, we have the above expression equals
\vspace{-6pt}
\begin{multline*}
    \E_{\boldsymbol{\sigma}}
    \Bigg[ 
     \E_{S_1^{\tilde{Z},\boldsymbol{\sigma}} \sim \mathcal{D}^{\frac{n}{3}}} \bigg[\sum_{i\in \{\boldsymbol{\sigma}(1),\dots,\boldsymbol{\sigma}(\frac{n}{3})\}}\1\left[\mathcal{B}(S_1^{\tilde{Z},\boldsymbol{\sigma}})(X_{i}) \neq h^*(X_{i}) \right] \bigg|\boldsymbol{\sigma} \bigg]
    +\\
    \sum_{i\notin \{\boldsymbol{\sigma}(1),\dots,\boldsymbol{\sigma}(\frac{n}{3})\}}\E_{S_1^{\tilde{Z},\boldsymbol{\sigma}} \sim \mathcal{D}^{\frac{n}{3}}} \left[\mathbb{P}_{X_i\sim \mathcal{D}}\left(\mathcal{B}(S_1^{\tilde{Z},\boldsymbol{\sigma}})(X_{i}) \neq h^*(X_{i}) \bigg| S_1^{\tilde{Z},\boldsymbol{\sigma}},\boldsymbol{\sigma} \right) \bigg|\boldsymbol{\sigma} \right] 
    \Bigg]
\end{multline*}
\begin{equation*}
   \hspace*{-8cm} \leq \E_{\boldsymbol{\sigma}}\left[\frac{n}{3}R(\frac{n}{3}) + \frac{5n}{3} R(\frac{n}{3})\right]
    =2nR(\frac{n}{3}).
\end{equation*}
as claimed.
\end{proof}

\vspace{-1.5em}
\subsection{Expected size of the conditional error region}
In the following lemma, we apply a Markov's inequality to derive a guarantee on the expected size of the conditional error regions of the base learner, based on the guarantee on the expected number of errors on a realizable supersample, estabilished in Lemma \ref{lemma:expected_error_iid}.
\begin{lemma}\label{expected_err_region_iid}
   Define a partition of $\sZ_x$ based on the conditional probability of error of the $k$-th base learner over random swaps given supersamples. Consider the following set for any $k\in [3]$, for some $j\in \mathbb{N}$:
   \begin{align*}
     R_{j}^{k,\tilde{Z}} &:= \left\{ X_i\in \sZ_x: \mathbb{P}_{\boldsymbol{\sigma}|\tilde{Z}}\left(\mathcal{B}\left(S_k^{\tilde{Z},\boldsymbol{\sigma}}\right)\left(X_i\right) \neq h^*\left(X_i\right)\big|\tilde{Z} \right)\in \left(2^{-j},2^{-j+1}\right]\right\}.
    \end{align*}
    We have a guarantee on the expected size of the $j$-th conditional error region of the $k$-th base learner over a random draw of the supersample as
    \[
    \mathbb{E}_{\tilde{Z}}\big[\big|R_{j}^{k,\tilde{Z}}\big|\big] \leq 2^j\cdot 2nR\left(\frac{n}{3}\right).
    \]
\end{lemma}
\begin{proof}
The number of data points with conditional probability of error in the range $\left(2^{-j},2^{-j+1}\right]$ is at most that with conditional probability of error at least $2^{-j}$ for any $j\in \mathbb{N}$. Thus, we have the expected size of the $j$-th conditional error region of $k$-th base learner bounded by
\begin{align*}
    \mathbb{E}_{\tilde{Z}}\big[\big|R_{j}^{k,\tilde{Z}}\big|\big]&= \mathbb{E}_{\tilde{Z}}\bigg[\sum_{X_i\in \sZ_x} \1\left[\mathbb{P}_{\boldsymbol{\sigma}| \tilde{Z}}\left(\mathcal{B}\left(S_k^{\tilde{Z},\boldsymbol{\sigma}}\right)\left(X_i\right) \neq h^*(X_i)\big|\tilde{Z}\right) \in \left(2^{-j},2^{-j+1}\right]\right]\bigg]\\
    &\leq \mathbb{E}_{\tilde{Z}}\bigg[\sum_{X_i\in \sZ_x} \mathbbm{1}\left[\mathbb{P}_{\boldsymbol{\sigma}| \tilde{Z}}\left(\mathcal{B}\left(S_k^{\tilde{Z},\boldsymbol{\sigma}}\right)\left(X_i\right) \neq h^*(X_i)|\tilde{Z}\right) > 2^{-j}\right]\bigg].
\end{align*}
Both the training data set $S_k^{\tilde{Z},\boldsymbol{\sigma}}$ and test data points $X_i \in \sZ_x$ are dependent on $\sZ$. With an application of Markov's inequality to the indicator function, we have that the above expression is bounded by
\begin{align*}
   &\mathbb{E}_{\tilde{Z}}\bigg[\sum_{X_i\in \sZ_x} 2^{j}\cdot \mathbb{P}_{\boldsymbol{\sigma}| \tilde{Z}}\left(\mathcal{B}\left(S_k^{\tilde{Z},\boldsymbol{\sigma}}\right)\left(X_i\right) \neq h^*(X_i)\big|\tilde{Z}\right)\bigg].
\end{align*}
By the linearity of conditional expectation and law of total expectation, we have the above expression equivalent to $2^j$ times the guarantee on the expected number of errors of the base learner over $\tilde{Z},\boldsymbol{\sigma}$ among $2n$ data samples, as follows:
\begin{align*}
    &2^j \cdot \mathbb{E}_{\tilde{Z}}\left[\E_{\boldsymbol{\sigma}| \tilde{Z}}\bigg[\sum_{X_i\in \sZ_x} \1\left[\mathcal{B}\left(S_k^{\tilde{Z},\boldsymbol{\sigma}}\right)\left(X_i\right) \neq h^*(X_i)\right]\Big|\tilde{Z}\bigg]\right]\\
    &= 2^j\cdot \E_{\tilde{Z}, \boldsymbol{\sigma}}\bigg[\sum_{X_i\in \sZ_x} \1 \left[\mathcal{B}(S^{\tilde{Z},\boldsymbol{\sigma}}_k)(X_i) \neq h^*(X_i)\right]\bigg].
\end{align*}
According to Lemma~\ref{lemma:expected_error_iid}, we have the above expression bounded by
$2^j\cdot 2nR(\frac{n}{3})$ as claimed.
\end{proof}
\vspace{-8pt}
\subsection{Distributional eCMI of a Majority-of-3 learner}
In the following lemma, we show that the distributional eCMI of a Majority-of-3 learner is of the same order as the i.i.d. guarantee on the expected number of errors of the base learner among $2n$ data samples in a realizable supersample. 
\vspace{-6pt}
\begin{lemma}\label{lemma:cond enty of maj-of-3}
The distributional eCMI of a Majority-of-3 learner $A$ given supersamples is
\begin{align*}
    \text{eCMI}_{\mathcal{D}}(\ell(A)) = I\left(\vec{\ell}(A(\sZ_{\boldsymbol{\sigma}}),\tilde{Z}); \boldsymbol{\sigma}|\tilde{Z}\right) = O\left(nR(\frac{n}{3})\right).
\end{align*}
\end{lemma}
\begin{proof}
For a deterministic learning algorithm, the conditional mutual information between the random selection of the training data set and the 0-1 loss vector evaluated on the supersample, is equal to conditional entropy of the 0-1 loss vector evaluated on the supersample. By the subadditivity of conditional entropy, we have
\begin{align*}
    &I\left(\vec{\ell}(A(\sZ_{\boldsymbol{\sigma}}),\tilde{Z}); \boldsymbol{\sigma}|\tilde{Z}\right) 
    = H\left(\vec{\ell}(A(\sZ_{\boldsymbol{\sigma}}),\tilde{Z}) | \tilde{Z}\right)
    \leq \sum_{i=1}^{2n}  H\left(\1\left[A(\sZ_{\boldsymbol{\sigma}})(X_i) \neq h^*(X_i)\right] \big| \sZ\right).
\end{align*}
Since the binary entropy function $h(p)= -p\log p - (1-p)\log (1-p)$ is bounded by the function $f(p) = -p\log p + 2p = p\log \frac{4}{p}$ on $[0,1]$, the above expression can be bounded by 
\begin{equation}\label{eq:dist_cond_entrpy}
    \sum_{i=1}^{2n} \E_{\tilde{Z}}\left[\mathbb{P}_{\boldsymbol{\sigma}| \tilde{Z}}(A(\sZ_{\boldsymbol{\sigma}})(X_i) \neq h^*(X_i) | \tilde{Z}) \log \frac{4}{ \mathbb{P}_{\boldsymbol{\sigma}| \tilde{Z}}(A(\sZ_{\boldsymbol{\sigma}})(X_i) \neq h^*(X_i) | \tilde{Z})}\right].
\end{equation}
For any $X_i \in \sZ_x$, we let $L_i(h^{\tilde{Z},\boldsymbol{\sigma}}_k)$ denote the event that $h^{\tilde{Z},\boldsymbol{\sigma}}_k$ errs at $X_i$,
$
L_i(h^{\tilde{Z},\boldsymbol{\sigma}}_k) = \big\{ h^{\tilde{Z},\boldsymbol{\sigma}}_k(X_i) \neq h^*(X_i)\big\}.
$
Since a majority-of-3 classifier errs when at least two of its base classifiers err, and $f(p)$ is monontone non-decreasing on $[0,1]$, we have \eqref{eq:dist_cond_entrpy} bounded by
\begin{equation}\label{eq:dist_entrpy_maj_of_3}
     \sum_{i=1}^{2n}\E_{\tilde{Z}}\Bigg[ \Big(\sum_{\substack{1 \leq k'<k \leq 3}}\mathbb{P}_{\boldsymbol{\sigma}| \tilde{Z}}(L_i(h^{\tilde{Z},\boldsymbol{\sigma}}_k)\wedge L_i(h^{\tilde{Z},\boldsymbol{\sigma}}_{k'}) | \tilde{Z})\Big) \log \frac{4}{ \min\{1,\sum_{\substack{1 \leq k'<k \leq 3}}\mathbb{P}_{\boldsymbol{\sigma}| \tilde{Z}}(L_i(h^{\tilde{Z},\boldsymbol{\sigma}}_k)\wedge L_i(h^{\tilde{Z},\boldsymbol{\sigma}}_{k'}) | \tilde{Z})\}}\Bigg].
\end{equation}
We define the conditional error region on the data samples in the supersample based on the conditional probability of error of the $k$-th base learner over $\boldsymbol{\sigma}$ given $\tilde{Z}$:
\begin{align*}
    R^{k,\tilde{Z}}_j := \left\{X_i \in \sZ_{x}: \mathbb{P}_{\boldsymbol{\sigma}|\sZ}(h_k^{\tilde{Z},\boldsymbol{\sigma}}(X_i) \neq h^*(X_i) | \tilde{Z}) \in (2^{-j}, 2^{-j+1}] \right\}.
\end{align*}
Let $j^{\sZ}_k(X_i)$ denote the index of the conditional error region such that the conditional probability of error of the $k$-th base learner at $X_i \in \sZ_x $, is in the interval $(2^{-j^{\sZ}_k(X_i)}, 2^{-j^{\sZ}_k(X_i)+1}]$. 
By the linearity of expectation and conditional independence of $S_k^{\sZ,\boldsymbol{\sigma}}, S_{k'}^{\sZ,\boldsymbol{\sigma}}$ given $\sZ$, we have \eqref{eq:dist_entrpy_maj_of_3} bounded by
\begin{align*}
    &\E_{\tilde{Z}}\Big[\sum_{X_i\in \sZ_x}
    (\sum_{\substack{1 \leq k'<k \leq 3}} 2^{-j^{\tilde{Z}}_k(X_i) -j^{\tilde{Z}}_{k'}(X_i)+2}) 
    \log \frac{4}{ \min\{1, \sum_{\substack{1 \leq k'<k \leq 3}} 2^{-j^{\tilde{Z}}_k(X_i)-j^{\tilde{Z}}_{k'}(X_i) +2 }\}}\Big]\\
    &\leq \sum_{\substack{1 \leq k'<k \leq 3}} \E_{\tilde{Z}}\Big[\sum_{X_i \in \sZ_x}
    (2^{-j^{\tilde{Z}}_k(X_i) -j^{\tilde{Z}}_{k'}(X_i)+2} ) 
    \log \frac{4}{2^{-j^{\tilde{Z}}_k(X_i)-j^{\tilde{Z}}_{k'}(X_i)+2}}\Big].
\end{align*}
Combined with the inequality $xy\leq \max\{x^2 , y^2\}$, we have that the above expression is bounded by
\begin{align*}
    & \sum_{\substack{1 \leq k'<k \leq 3}} \E_{\tilde{Z}}\Big[\sum_{X_i \in \sZ_x }
    \max\{ 2^{-2j^{\tilde{Z}}_k(X_i)+2} , 2^{-2j^{\tilde{Z}}_{k'}(X_i)+2} \}
    \log \frac{4}{ \max\{ 2^{-2j^{\tilde{Z}}_k(X_i)+2} , 2^{-2j^{\tilde{Z}}_{k'}(X_i)+2} \} }\Big].
\end{align*}
And the function $f(p)= p\log \frac{4}{p}$ is non-negative on $[0,1]$, we have the above expression bounded by
\begin{align*}
    &\sum_{\substack{1 \leq k'<k \leq 3}}  \E_{\tilde{Z}}\Big[\sum_{X_i \in \sZ_x }
    2^{-2j^{\tilde{Z}}_k(X_i)+2} \log \frac{4}{2^{-2j^{\tilde{Z}}_k(X_i)+2}} +
     2^{-2j^{\tilde{Z}}_{k'}(X_i)+2} \log \frac{4}{2^{-2j^{\tilde{Z}}_{k'}(X_i)+2}}\Big]\\
    &\leq \sum_{\substack{1 \leq k'<k \leq 3}}\Big( \E_{\tilde{Z}}[\sum_{j=1}^{\infty}|R^{k,\tilde{Z}}_j|\cdot
    2^{-2j+2}\cdot 2j] +
     \E_{\tilde{Z}}[\sum_{j=1}^{\infty}|R^{k',\tilde{Z}}_j|\cdot
    2^{-2j+2}\cdot 2j] \Big)\\
    &=  16 \sum_{k=1}^3 \sum_{j=1}^{\infty}\E_{\tilde{Z}}[|R^{k,\tilde{Z}}_j|]\cdot
    2^{-2j}\cdot j 
\end{align*}
 According to Lemma \ref{expected_err_region_iid}, we have
    $\E_{\sZ}[|R^{k,\sZ}_j|] \leq 2^j\cdot 2nR(\frac{n}{3})$, for any $k\in [3]$.
Thus, the above expression is bounded by
\begin{align*}
    48\cdot \Big(\sum_{j=1}^{\infty} j2^{-j}\Big) \cdot 2nR\left(\frac{n}{3}\right)
    = 192 \cdot nR\left(\frac{n}{3}\right)
    =O\Big(nR\left(\frac{n}{3}\right)\Big).
\end{align*}
This concludes the proof.
\end{proof}
\subsection{Expected training error rate}
In the following lemma, we derive a guarantee on the expected training error rate of the Majority-of-3 learner with respect to a random draw of the $n$ training samples $\sZ_{\boldsymbol{\sigma}}$. Without loss of generality, we may assume that $\boldsymbol{\sigma}(i) = i$ for all $i\in [n]$. 
\begin{lemma}
    The expected training error rate of the Majority-of-3 learner over a random draw of $n$ training samples is,
$$
\E_{\sZ_{\boldsymbol{\sigma}}}\Big[\ell(A(\sZ_{\boldsymbol{\sigma}}), \sZ_{\boldsymbol{\sigma}})\Big] \leq 2 \cdot R\left(\frac{n}{3}\right).
$$

\end{lemma}
\begin{proof}
Since every i.i.d. sequence is exchangeable, the expected training error rate of $A$ equals the expected error rate of $A$ at an arbitrary training sample. Without loss of generality, we may assume that the test sample is $(X_1,Y_1) \in \sZ_{\boldsymbol{\sigma}}$, then we have
\begin{align*}
    \E_{\sZ_{\boldsymbol{\sigma}}}\Big[\ell(A(\sZ_{\boldsymbol{\sigma}}), \sZ_{\boldsymbol{\sigma}})\Big] 
    &= \E_{\sZ_{\boldsymbol{\sigma}}}\Big[\frac{1}{n}\sum_{i=1}^n\ell(A(\sZ_{\boldsymbol{\sigma}}), \sZ_{\boldsymbol{\sigma}(i)})\Big]\\
    &= \E_{\sZ_{\boldsymbol{\sigma}}}\Big[\ell(A(\sZ_{\boldsymbol{\sigma}}), \sZ_{1})\Big]\\
    &= \E_{\sZ_{\boldsymbol{\sigma}}}\Big[ \1[A(\sZ_{\boldsymbol{\sigma}})(X_1)\neq h^*(X_1)] \Big].
\end{align*}
    We observe that if the majority-of-3 classifier errs at $X_1$, then at least one of the other two base classifiers must err.
    Thus, for $n$ fixed training samples $\sZ_{\boldsymbol{\sigma}}$,
    we have
    \begin{align*}
        \1\left[A(\sZ_{\boldsymbol{\sigma}})(X_1)\neq h^*(X_1)\right] &\leq
         \1\left[h_2^{\sZ,\boldsymbol{\sigma}}(X_1)\neq h^*(X_1) \lor h_3^{\sZ,\boldsymbol{\sigma}}\neq h^*(X_1)\right]\\
        &\leq 
        \1\left[h_2^{\sZ,\boldsymbol{\sigma}}(X_1)\neq h^*(X_1)\right] + \1\left[h_3^{\sZ,\boldsymbol{\sigma}}(X_1)\neq h^*(X_1)\right].
    \end{align*}
    Since $S_2^{\sZ,\boldsymbol{\sigma}}, S_3^{\sZ,\boldsymbol{\sigma}}$ are sampled independently of the test sample $X_1\sim D$,
    we have the expected training error rate of the Majority-of-3 learner bounded by the expected true error rate of the other two base learners. That is,
    \begin{align*}
        &\E_{\sZ_{\boldsymbol{\sigma}}\sim D^n}\Big[ \1[A(\sZ_{\boldsymbol{\sigma}})(X_1)\neq h^*(X_1)] \Big]\\
        &\leq\E_{\sZ_{\boldsymbol{\sigma}} \sim D^n} \left[\1[\mathcal{B}\left(S_2^{\tilde{Z},\boldsymbol{\sigma}}\right)(X_1)\neq h^*(X_1)]\right] + \E_{\sZ_{\boldsymbol{\sigma}} \sim D^n} \Big[\1[\mathcal{B}\left(S_3^{\tilde{Z},\boldsymbol{\sigma}}\right)(X_1)\neq h^*(X_1)]\Big] \\
        &= \E_{S_2^{\sZ,\boldsymbol{\sigma}}\sim \mathcal{D}^{\frac{n}{3}}}\Big[\mathbb{P}_{X_1\sim \mathcal{D}}\left(\mathcal{B}\left(S_2^{\tilde{Z},\boldsymbol{\sigma}}\right)(X_1)\neq h^*(X_1)\big| S_2^{\sZ,\boldsymbol{\sigma}} \right)\Big] + \E_{S_3^{\sZ,\boldsymbol{\sigma}}\sim \mathcal{D}^{\frac{n}{3}}}\Big[\mathbb{P}_{X_1\sim \mathcal{D}}\left(\mathcal{B}\left(S_3^{\tilde{Z},\boldsymbol{\sigma}}\right)(X_1)\neq h^*(X_1) \big| S_3^{\sZ,\boldsymbol{\sigma}}\right)\Big]\\
        &\leq 2\cdot R(\frac{n}{3}).
    \end{align*}
This completes the proof of Theorem \ref{thm:distributional eCMI(A)}.
\end{proof}

\section{Conclusions}
We provide a bound on the worst-case eCMI of a randomized Majority-of-5 learner. It is the first proof that an algorithm can have eCMI of order $O(d)$, where $d$ is the VC-dimension of the concept class. 
Combining with the guarantee on the expected training error rate of the algorithm, it recovers an optimal PAC generalization guarantee for learning VC classes in the realizable case.
We also study the distributional eCMI and show that a Majority-of-3 base learners with an optimal in-expectation generalization guarantee suffices to provide an optimal guarantee on the distributional eCMI.

We also note that all of our results immediately extend to multiclass learning \cite{10.5555/2789272.2912074}, where $d$ is the Daniely–Shalev-Shwartz (DS) dimension $d_{DS}$ \cite{daniely2014optimal}. Indeed, since the loss function is the 0-1 loss for multiclass learning as well, and the definition of eCMI and all of our proofs only involve the binary losses, the proofs remain valid for multiclass learners as well.
Therefore, Theorem \ref{thm:eCMI(A)} and Theorem \ref{thm:distributional eCMI(A)} also hold for multiclass learning in the realizable case.
In particular, \cite{pabbaraju2026optimalsamplecomplexitymulticlass} has shown that there exists an one-inclusion-graph algorithm that achieves the optimal expected error rate $R(m)$ of order $O(\frac{d_{DS}}{m})$ for multiclass learning in the realizable case. Plugging this $R(m)$ into Theorem \ref{thm:eCMI(A)} then implies that our randomized Majority-of-5 learner yields an optimal worst-case eCMI of order $O(d_{DS})$.

\section*{Acknowledgments}
Steve Hanneke acknowledges support by grant no.\ 2024243 from the United States - Israel Binational Science Foundation (BSF).

\bibliography{refs}
\bibliographystyle{alpha} 

\end{document}